%% file: neurips_2026.tex
\documentclass{article}

\usepackage[preprint]{neurips_2026}

\usepackage[utf8]{inputenc} 
\usepackage[T1]{fontenc}    
\usepackage{hyperref}       
\usepackage{url}            
\usepackage{booktabs}       
\usepackage{amsfonts}       
\usepackage{nicefrac}       
\usepackage{microtype}      
\usepackage{xcolor}         

\usepackage{empheq}
\usepackage{bm}
\usepackage{amsthm}
\usepackage{pifont}
\usepackage{multirow} 
\usepackage{graphicx}
\usepackage{subcaption}
\usepackage{booktabs}
\usepackage{wrapfig}
\usepackage{colortbl}

\newtheorem{proposition}{Proposition}
\theoremstyle{remark}
\newtheorem{remark}{Remark}
\theoremstyle{plain}
\newtheorem{lemma}{Lemma}
\definecolor{linkpink}{RGB}{219, 48, 122}
\definecolor{citepink}{RGB}{219, 48, 122}
\hypersetup{
    colorlinks=true,
    linkcolor=linkpink,
    citecolor=citepink,
    urlcolor=linkpink
}

\title{Spectral-Sphere-Constrained Hyper-Connections}

\author{Zhaoyi Liu\textsuperscript{1} \quad Haichuan Zhang\textsuperscript{2} \quad Ang Li\textsuperscript{1}\\
\textsuperscript{1}University of Maryland, College Park \\
\textsuperscript{2}University of Utah \\
{\tt\small \textsuperscript{1}\{zhaoyil, angliece\}@umd.edu} \\ \tt\small \textsuperscript{2}\{hc.zhang\}@utah.edu}

\begin{document}

\maketitle

\input{section/1_abstract}

\input{section/2_intro}
\input{section/3_preliminary}
\input{section/4_analysis}

\input{section/5_method}
\input{section/6_experiment}
\input{section/7_conclusion}

\bibliography{refs}
\bibliographystyle{unsrt}

\input{section/appendix}



\end{document}

%% file: section/1_abstract.tex
\begin{abstract}
Hyper-Connections (HC) extend residual connections into multiple streams, employing residual matrices for cross-stream mixing to enrich model expressivity. However, unconstrained mixing disrupts the identity mapping property intrinsic to the residual connection, causing unstable training. To address this, Manifold-Constrained Hyper-Connections (mHC) and its variants restrict these matrices to be doubly stochastic via Sinkhorn-Knopp (SK) algorithm or permutation-based parameterizations. We reveal three limitations of this doubly stochastic constraint: 
(1) identity degeneration, where learned matrices collapse around the identity initialization and diminish cross-stream interactions,
(2) a expressivity bottleneck, where the doubly stochastic constraint restricts the freedom of the subdominant spectrum of the residual matrices, preventing the model from selectively preserving or attenuating cross-stream variations, and (3) parameterization inefficiencies, manifesting as unstable SK iterations or the factorial-scaling overhead of permutation-based parameterizations. 
To overcome these flaws, we propose Spectral-Sphere-Constrained Hyper-Connections ($\mathrm{s}^{2}$HC\footnote{Code is available at \url{https://github.com/6zHAOyi/s2HC}.}). By confining residual matrices to a spectral norm sphere, $\mathrm{s}^{2}$HC restores free control over the subdominant spectrum, enabling the model to selectively preserve or attenuate cross-stream variations. This shift eliminates unstable SK iterations and factorial parameterization, enabling expressive, non-degenerate residual matrices while preserving training stability.
\end{abstract}
\vspace{-1.5em}
\begin{figure}[h]
    \centering
    \begin{minipage}[b]{0.25\textwidth}
        \centering
        \includegraphics[width=\linewidth]{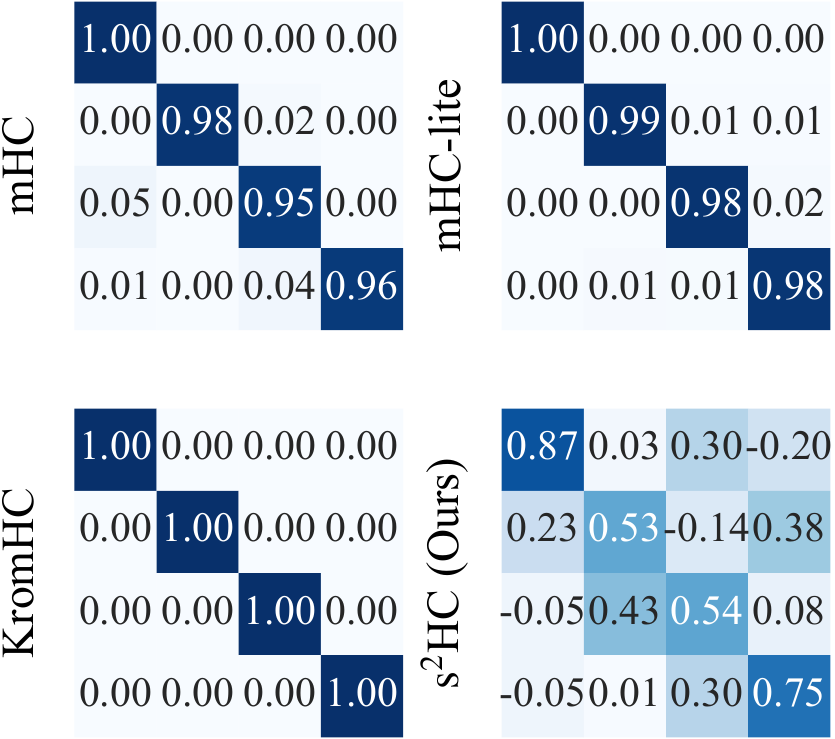}
    \end{minipage}
    \begin{minipage}[b]{0.31\textwidth}
        \centering
        \includegraphics[width=\linewidth]{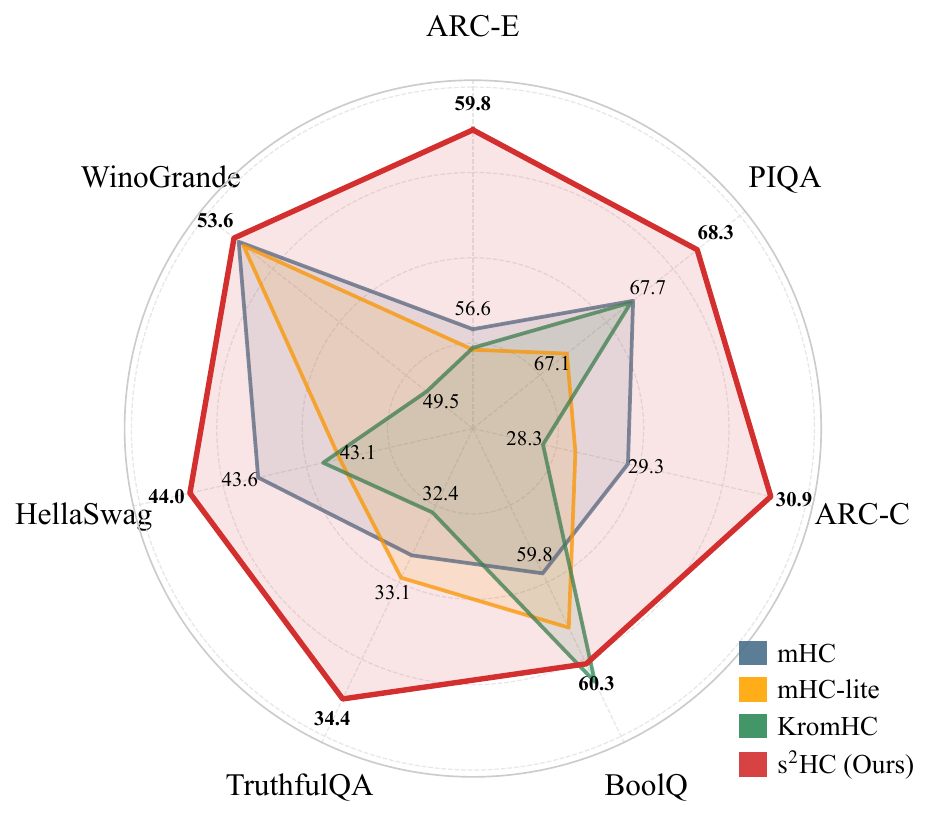}
    \end{minipage}
    \begin{minipage}[b]{0.23\textwidth}
    \centering
        \includegraphics[width=\linewidth]{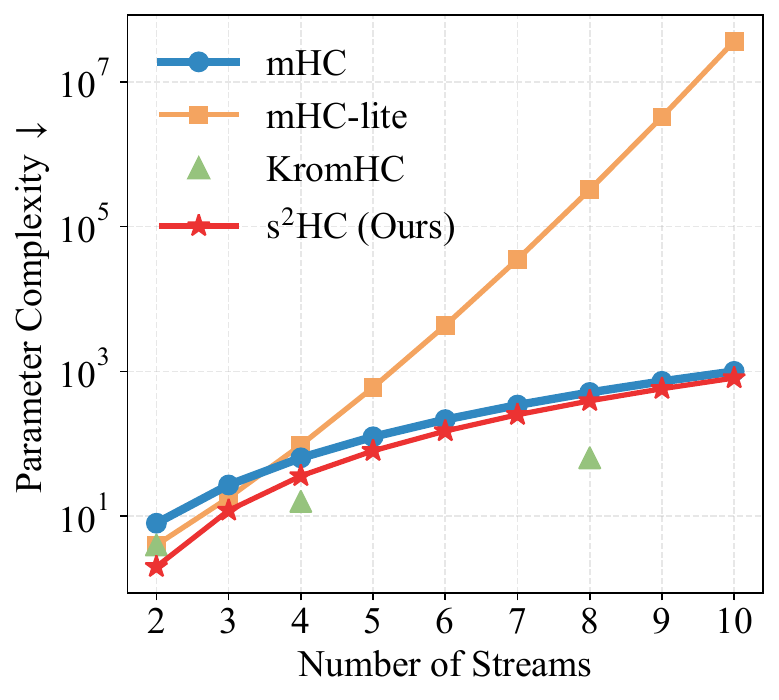}
    \end{minipage}
    \caption{$\mathrm{s}^{2}$HC addresses the identity degeneration, expressivity bottleneck, and parameterization inefficiencies in existing Manifold-Constrained Hyper-Connections (mHC~\cite{xie2025mhc}) and its permutation-based variants (mHC-lite~\cite{yang2026mhc} and KromHC~\cite{zhou2026kromhc}). \textbf{Left}: Learned residual matrices (4 streams). mHC, mHC-lite, and KromHC degenerate into the identity mapping, whereas $\mathrm{s}^{2}$HC leverages expressive, non-degenerate residual matrices. \textbf{Middle}: Model performance, with each axis scaled to its own range and absolute scores annotated. $\mathrm{s}^{2}$HC improves over baselines across benchmarks. \textbf{Right}: Parameterization overhead, measured by the parameter complexity of the residual parameterization as the number of streams grows. $\mathrm{s}^{2}$HC scales compactly, avoiding both the factorial explosion of mHC-lite and the restriction of KromHC to specific stream counts.}
    \label{fig:teaser}
\end{figure}

%% file: section/2_intro.tex
\section{Introduction}
Residual connections~\cite{he2016deep} have been a cornerstone of deep learning for over a decade, stabilizing gradient propagation through identity mappings and becoming a standard component of modern deep networks~\cite{liu2024deepseek,touvron2023llama,brown2020language}. Hyper-Connections (HC)~\cite{zhu2024hyper} have recently extended this single-stream residual connection into parallel residual streams by using a dynamic residual matrix at each layer to mix the features across the streams, increasing the model's topological complexity and capacity. However, unconstrained residual matrices compromise the identity preservation property intrinsic to the residual connection, which causes training instability~\cite{xie2025mhc}. 

To address this, Manifold-Constrained Hyper-Connections (mHC)~\cite{xie2025mhc} proposes constraining residual matrices to be doubly stochastic to preserve the identity mapping property. Doubly stochastic matrices are characterized by non-negative entries and unit row and column sums. This constraint bounds the spectral norm by one to mitigate gradient explosion, while preserving the mean component of the residual streams~\cite{xie2025mhc}. mHC enforces this constraint via Sinkhorn-Knopp (SK) ~\cite{sinkhorn1967concerning} algorithm, which iteratively normalizes rows and columns of residual matrices. However, SK yields only an unstable constraint approximation. As reported in~\cite{yang2026mhc}, the resulting constraint violations accumulate across depth, undermining stability.

A recent variant, mHC-lite~\cite{yang2026mhc}, enforces exact doubly stochasticity by parameterizing residual matrices as convex combinations of permutation matrices. However, this introduces factorial growth in auxiliary parameters, since the number of permutation matrices scales factorially with the number of streams, limiting scalability~\cite{yang2026mhc}. KromHC~\cite{zhou2026kromhc} mitigates this overhead by constructing the residual matrix via Kronecker products of smaller doubly stochastic matrices, each parameterized similarly to mHC-lite. However, this results in an incomplete parameterization of the doubly stochastic manifold. Moreover, its efficiency relies on decomposing the number of residual streams into small factors, and degrades to intractable factorial complexity when large prime factors are present~\cite{zhou2026kromhc}.

Beyond these parameterization inefficiencies, we identify two previously overlooked intrinsic limitations of the doubly stochastic constraint adopted in these methods. First, it is prone to \textbf{\textit{identity degeneration}}, where the learned residual matrices in each layer concentrate around the identity initialization, practically abandoning the intended cross-stream interactions. Second, it imposes a \textbf{\textit{spectral expressivity bottleneck}}: the doubly stochastic constraint restricts the freedom of the subdominant spectrum of the residual matrices, limiting the model's ability to selectively preserve or attenuate cross-stream variations while mixing. (Detailed analysis is in \S~\ref{sec:observation}.)

To overcome these limitations, we propose Spectral-Sphere-Constrained Hyper-Connections ($\mathrm{s}^{2}$HC). We constrain the residual matrices to a spectral norm sphere rather than to a doubly stochastic manifold. As shown in Figure~\ref{fig:teaser}, this shift yields three key advantages: \ding{182} $\mathrm{s}^{2}$HC restores subdominant spectrum control, yielding non-degenerate residual matrices. \ding{183} These non-degenerate matrices selectively preserve or attenuate cross-stream variations, supporting expressive feature interactions and improving the model performance. \ding{184} $\mathrm{s}^{2}$HC provides an exact closed-form parameterization with stable spectral norm control, eliminating both the factorial overhead and the approximation errors of prior methods. In summary, our contributions are:
\begin{itemize}
     \item We introduce Spectral-Sphere-Constrained Hyper-Connections ($\mathrm{s}^{2}$HC), which reformulates hyper-connection constraints within a spectral norm sphere, mitigating identity degeneration, alleviating the expressivity bottleneck, and eliminating factorial parameterization overhead.
     \item Experiments demonstrate that $\mathrm{s}^{2}$HC improves model capability over existing doubly stochastic constrained hyper-connection methods while preserving the training stability.
     \item The expressivity, stability, and parameterization efficiency of $\mathrm{s}^{2}$HC offer the community a practical reference paradigm for deploying hyper-connection architectures.
\end{itemize}

%% file: section/3_preliminary.tex
\section{Related Work}
Residual connections~\cite{he2016deep} stabilize deep network training by introducing identity skip connections. Expanding on this, Hyper-Connections (HC)~\cite{zhu2024hyper} introduce parallel residual streams mixed by dynamic matrices to enhance model capacity. Frac-Connections~\cite{zhu2025frac} fragment streams into chunks as an alternative extension for Hyper-Connections to reduce the memory access costs of parallel residual streams. 
To address the training instability inherent in HC, Manifold-Constrained Hyper-Connections (mHC~\cite{xie2025mhc}) and its exact-stochasticity variants (mHC-lite~\cite{yang2026mhc}, KromHC~\cite{zhou2026kromhc}) confine the residual matrices to be doubly stochastic. While effective for restoring stability, enforcing this constraint inevitably introduces parameterization inefficiencies and intrinsic expressivity bottlenecks. In this work, we align with the hyper-connection paradigm, focusing on resolving these expressivity and parameterization bottlenecks.

\section{Preliminary}
\label{sec:prelim}
\textbf{Hyper-Connections (HC).} 
To enrich the expressivity of the model, Hyper-Connections~\cite{zhu2024hyper} extend the single-stream paradigm to $n$ parallel residual streams. Let $X_l=(\bm{x}_{l,1}^\top,\bm{x}_{l,2}^\top,...,\bm{x}_{l,n}^\top)^\top\in \mathbb{R}^{n \times C}$ represent the expanded features of $n$ streams at the $l$-th layer. HC introduces a dynamic mixing mechanism:
\begin{equation}
    X_{l+1}=\mathcal{H}_l^{\mathrm{res}}X_l+(\mathcal{H}_l^{\mathrm{post}})^\top \mathcal{F}(\mathcal{H}_l^{\mathrm{pre}}X_l,\mathcal{W}_l)
    \label{eq:HC}
\end{equation}
Here, $\mathcal{F}(\cdot,\mathcal{W}_l)$ denotes the learnable layer transformation with parameter $\mathcal{W}_l$ in the branch (e.g., Attention or MLP block). $\mathcal{H}_l^{\mathrm{res}}\in
\mathbb{R}^{n\times n}$ represents a learnable residual matrix that mixes features within the residual streams. Similarly, $\mathcal{H}_l^{\mathrm{pre}}\in
\mathbb{R}^{1\times n}$ aggregates features from
the $(n\times C)$-dim stream into a $(1\times C)$-dim layer branch input, and conversely, $\mathcal{H}_l^{\mathrm{post}}\in
\mathbb{R}^{1\times n}$ maps the layer branch output back onto the streams. However, without constraints on these residual matrices, HC may suffer from severe training instability~\cite{xie2025mhc}.

\textbf{Manifold-Constrained Hyper-Connections (mHC).}
To ensure training stability, mHC~\cite{xie2025mhc} constrains $\mathcal{H}_l^{\mathrm{res}}$ to lie in the Birkhoff polytope $\mathcal{B}_n$, i.e., the set of doubly stochastic matrices. This polytope is a subset of the affine subspace $\mathcal{A}_n=\{\mathcal{H}\in\mathbb{R}^{n\times n} \mid 
\mathcal{H}\mathbf{1}_n=\mathbf{1}_n,\ 
\mathbf{1}_n^\top\mathcal{H}=\mathbf{1}_n^\top \}$, which is obtained by adding the element-wise non-negativity constraint to $\mathcal{A}_n$:
\begin{equation}
\mathcal{B}_n = \{\mathcal{H}\in\mathcal{A}_n \mid 
\mathcal{H} \ge 0 \}.
\end{equation}
The affine constraint in $\mathcal{A}_n$ ensures that the uniform vector $\mathbf{1}_n$ remains invariant, thereby conserving the mean component across residual streams. Imposing the additional non-negativity condition bounds the spectral norm to 1 (i.e., $\|\mathcal{H}_l^{\mathrm{res}}\|_2 = 1$), which prevents signal amplification and stabilizes deep propagation~\cite{xie2025mhc}.

As such, mHC can be fully formulated as:
\begin{empheq}[left=\empheqlbrace]{equation}
\begin{aligned}
&\bm{x}_l'=\text{RMSNorm}(\bm{x}_l) \\
&\mathcal{H}_l^{\mathrm{pre}}=\operatorname{sigmoid}\big(\alpha_l^{\mathrm{pre}}\cdot(\bm{x}_l'W_l^{\mathrm{pre}})+ \bm{b}_l^{\mathrm{pre}}\big) \\
&\mathcal{H}_l^{\mathrm{post}}=2\cdot\operatorname{sigmoid}\big(\alpha_l^{\mathrm{post}}\cdot(\bm{x}_l'W_l^{\mathrm{post}})+ \bm{b}_l^{\mathrm{post}}\big) \\
&\mathcal{H}_l^{\mathrm{res}}=\text{SK}\big(\exp(\alpha_l^{\mathrm{res}}\cdot \text{mat}(\bm{x}_l'W_l^{\mathrm{res}})+ \bm{b}_l^{\mathrm{res}})\big)
\end{aligned}
\label{eq:mHC}
\end{empheq}
where $\bm{x}_l\in\mathbb{R}^{1\times nC}$ is the vector flattened from the expanded streams $X_l$. $W_l^{\mathrm{pre}},W_l^{\mathrm{post}}\in \mathbb{R}^{nC\times n}$ and $W_l^{\mathrm{res}}\in\mathbb{R}^{nC\times n^2}$ are linear projections for dynamic mappings. The terms $\bm{b}_l^{\mathrm{pre}},\bm{b}_l^{\mathrm{post}}\in\mathbb{R}^{1\times n}$ and $\bm{b}_l^{\mathrm{res}}\in\mathbb{R}^{n\times n}$ are learnable biases. $\alpha_l^{\mathrm{pre}},\alpha_l^{\mathrm{post}},\alpha_l^{\mathrm{res}}$ are scalar gating factors. $\text{RMSNorm}(\cdot)$ refers to the RMSNorm~\cite{zhang2019root}. $\operatorname{sigmoid}(\cdot)$ denotes the Sigmoid function. $\text{mat}(\cdot)$ reshapes a matrix from $\mathbb{R}^{1\times n^2}$ to $\mathbb{R}^{n\times n}$ while $\text{SK}(\cdot)$ denotes the Sinkhorn–Knopp algorithm which iteratively alternates row-wise and column-wise normalization to enforce approximate doubly stochasticity. 
However, finite SK iterations cannot guarantee exact doubly stochasticity. The resulting approximation errors can accumulate across layers, potentially undermining the stability of deep networks. 


%% file: section/4_analysis.tex
\section{Observation}
\label{sec:observation}
\begin{figure}[t]
    \centering
    \begin{minipage}[c]{0.33\linewidth}
        \centering
        \includegraphics[width=\linewidth]{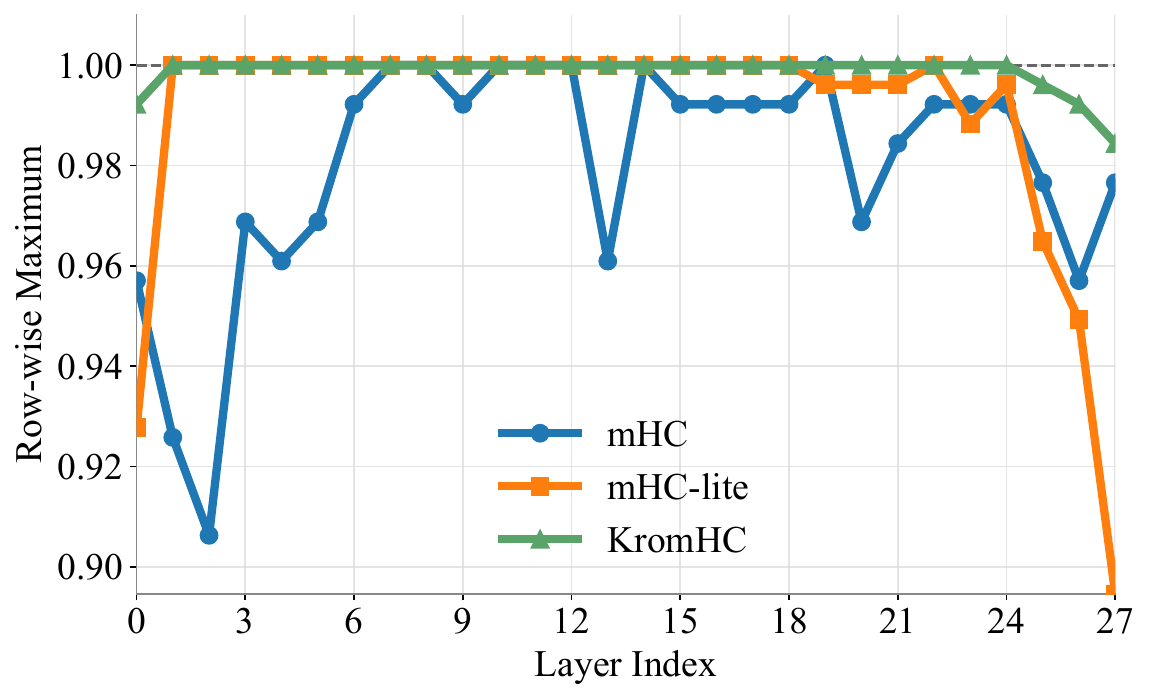}
    \end{minipage}
    \begin{minipage}[c]{0.33\linewidth}
        \centering
        \includegraphics[width=\linewidth]{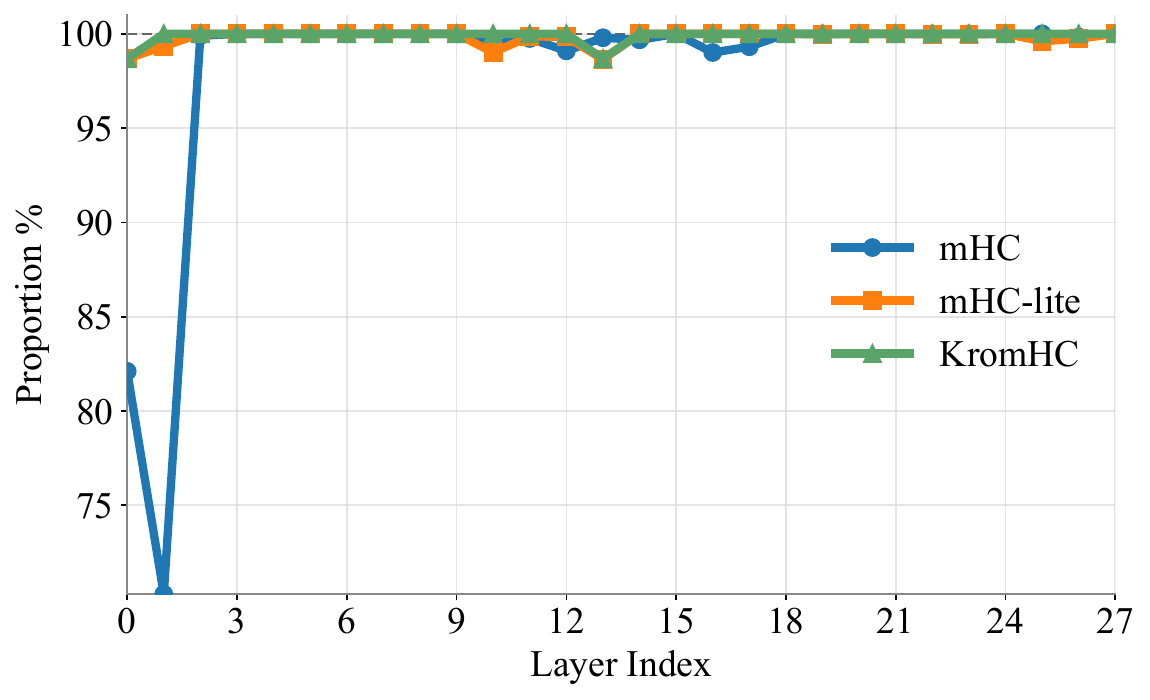}
    \end{minipage}
    \caption{Depth-wise analysis of $\mathcal{H}_l^{\mathrm{res}}$ for mHC, mHC-lite and KromHC after training. \textbf{Left}: the median row-wise maximum entries of $\mathcal{H}_l^{\mathrm{res}}$. \textbf{Right}: the proportion of $\mathcal{H}_l^{\mathrm{res}}$ where all row maximums are on the diagonal.}
    \label{fig:H_observation}
    \vspace{-1.5em}
\end{figure}
To evaluate the practical expressivity of these doubly stochastic constrained hyper-connections (mHC, mHC-lite, and KromHC), we analyze their learned residual matrices $\mathcal{H}_l^\mathrm{res}$ and pairwise cosine similarity among residual streams after being mixed by $\mathcal{H}_l^\mathrm{res}$ across all layers of the model. The experiments are conducted on a 28-layer 0.75B Qwen3 model with untied word embedding~\cite{qwen3technicalreport}.

\subsection{Identity Degeneration}
\label{sec:identity_degeneration}
\begin{wrapfigure}{r}{0.33\linewidth}
    \centering
    \includegraphics[width=\linewidth]{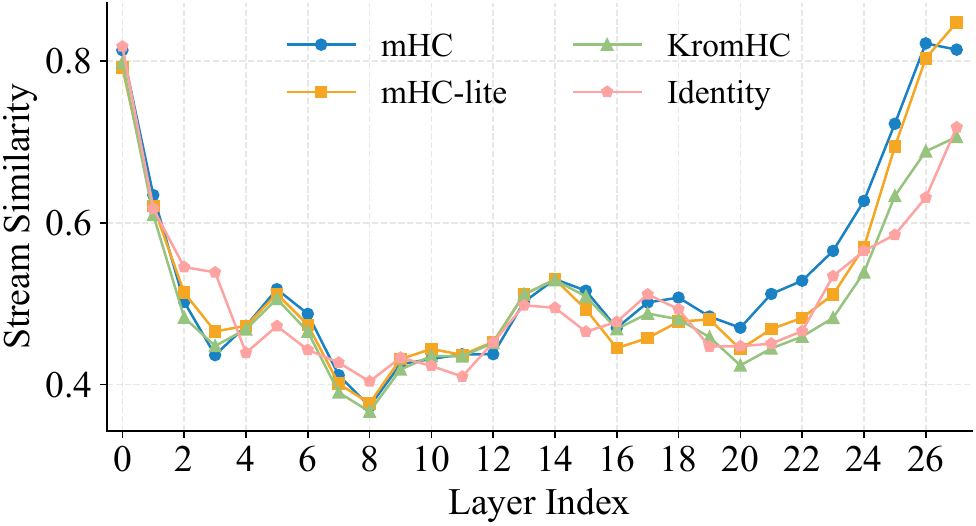}
    \caption{Mean pairwise cosine similarity among residual streams after being mixed by $\mathcal{H}_l^{\mathrm{res}}$ across model layers. Identity denotes the baseline where $\mathcal{H}_l^{\mathrm{res}}$ is fixed as an identity matrix while keeping all other settings identical to mHC, mHC-lite, and KromHC.}
    \label{fig:obs_stream_sim}
\end{wrapfigure}

As shown in Figure~\ref{fig:H_observation}, for all of these methods, the median row-wise maximum of the learned residual matrices $\mathcal{H}_l^{\mathrm{res}}$ remains tightly concentrated around 1. All three methods exhibit minimal dispersion, with the widest range only spanning 0.98–1.0 for KromHC, 0.90-1.0 for mHC-lite and mHC. Moreover, the fraction of matrices whose row maxima lie strictly on the diagonal stays consistently high (above 99\% for mHC-lite, KromHC).
Since $\mathcal{H}_l^{\mathrm{res}}$ is doubly stochastic (non-negative with unit row and column sums), the concentration of row-wise maxima on the diagonal with values approaching 1 implies that the learned matrices degenerate around the identity as shown in Figure~\ref{fig:teaser}, up to small off-diagonal tails. This indicates that the model actually abandons active cross-stream interactions at each layer. (More statistics and visualizations of $\mathcal{H}_l^{\mathrm{res}}$ are provided in Appendix~\ref{app:expressivity}.)

This degeneration is also reflected in the pairwise similarity within the residual streams. As illustrated in Figure~\ref{fig:obs_stream_sim}, the similarity trajectories of mHC, mHC-lite, and KromHC closely follow those of the identity mapping baseline (where $\mathcal{H}_l^{\mathrm{res}}$ is fixed as an identity matrix while keeping all other settings identical to these methods). This indicates that feature evolution is predominantly governed by the nonlinear branches rather than the residual feature mixing. This degeneration may explain the resulting training stability of these doubly stochastic constrained hyper-connections. However, prioritizing this form of apparent stability runs counter to the original design motivation of hyper-connections, which aim to enhance the model’s topological expressivity through active cross-stream feature interactions~\cite{xie2025mhc,zhu2024hyper}.

\subsection{Spectral Expressivity Bottleneck}
As discussed above, we observe that these doubly stochastic constrained hyper-connections degenerate into near-identity mappings, limiting cross-stream mixing. We suggest that this degeneration may be an inherent consequence of the spectral property of the doubly stochastic constraint.

As established by~\cite{liu2026homogeneity}, under standard mixing conditions, doubly stochastic operators exhibit a strong inductive bias toward uniformity: while the mean component is preserved, the subdominant spectral modes are progressively attenuated. Consequently, repeated cross-stream mixing suppresses variations orthogonal to the mean, leading to feature homogenization.

Under this perspective, \textbf{\textit{our key insight}} is that this suppression creates a spectral expressivity bottleneck for hyper-connections. When a doubly stochastic residual matrix departs from the identity to perform cross-stream mixing, its restricted subdominant singular values attenuate the off-mean directions being mixed. Thus, each layer faces a dilemma: the constraint biases mixing toward attenuation, limiting the model's ability to choose which cross-stream variations to preserve. The model may consequently remain near its identity initialization to avoid such attenuation, as observed in \S~\ref{sec:identity_degeneration}. This renders cross-stream interaction effectively inert: hyper-connections retain their dedicated design but contribute little, quietly reducing to ordinary residual connections.

%% file: section/5_method.tex
\section{Methodology}

Motivated by these observations, we propose Spectral-Sphere-Constrained Hyper-Connections ($\mathrm{s}^{2}$HC). Instead of confining the residual matrix $\mathcal{H}_l^{\mathrm{res}}$ to be doubly stochastic, we reformulate it into a spectral norm sphere restricted to the affine subspace $\mathcal{A}_n$. 
\subsection{Affine-Constrained Spectral Sphere}
\begin{wrapfigure}{l}{0.41\textwidth}
\vspace{-1em}
  \centering
  \includegraphics[width=0.38\textwidth]{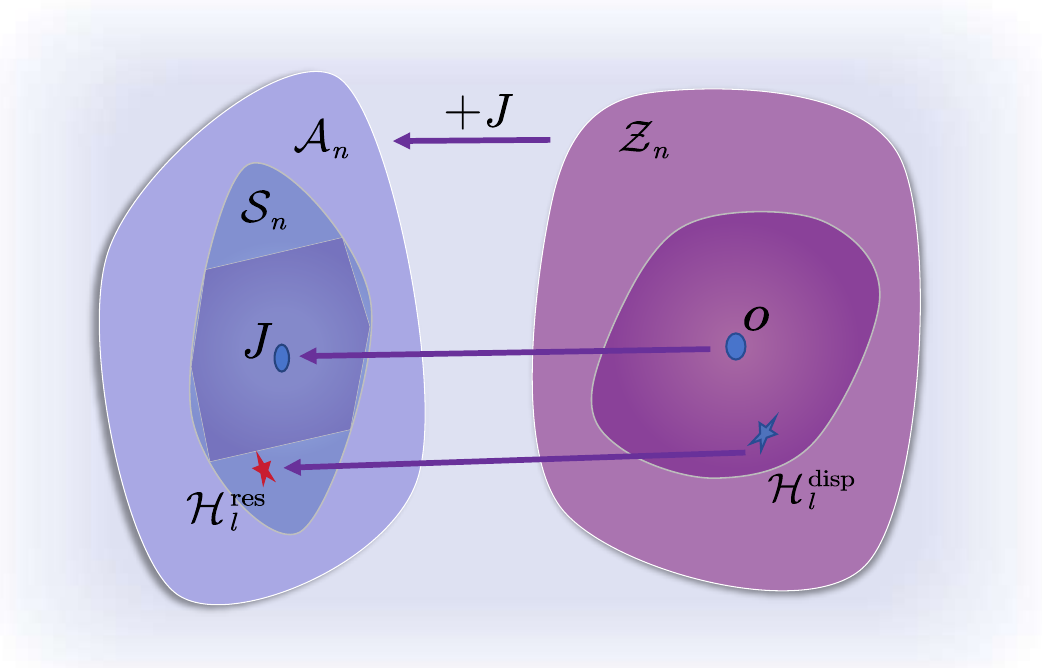}
  \caption{Overview of Spectral-Sphere-Constrained Hyper-Connections ($\mathrm{s}^{2}$HC). The right plane depicts the zero-marginal subspace $\mathcal{Z}_n$, where the SVD parameterization generates $\mathcal{H}_l^{\mathrm{disp}}$ under the bounded spectral norm $\|\mathcal{H}_l^{\mathrm{disp}}\|_2 \le 1$, centered at the origin $O$. The left plane depicts the target affine subspace $\mathcal{A}_n$, where the Birkhoff polytope $\mathcal{B}_n$ (inner polygon) is contained in the affine-constrained spectral norm sphere $\mathcal{S}_n$. The affine shift $+J$ maps $O$ to $J$, and $\mathcal{H}_l^{\mathrm{disp}}$ to the residual matrix $\mathcal{H}_l^{\mathrm{res}}$.}
\label{fig:overview}
\vspace{-1.5em}
\end{wrapfigure}
We define an affine-constrained spectral norm sphere $\mathcal{S}_n = \{\mathcal{H} \in \mathcal{A}_n \mid \|\mathcal{H}\|_2 = 1\}$. Constraining the residual matrix $\mathcal{H}_l^{\mathrm{res}}$ within $\mathcal{S}_n$ guarantees three critical properties:
\begin{enumerate}
    \item \textbf{Mean Preservation}. $\mathcal{S}_n$ resides in $\mathcal{A}_n$, where the affine constraint ensures that the uniform vector $\mathbf{1}_n$ remains an invariant direction, thereby conserving the mean component across residual streams.
    \item \textbf{Spectral Stability}. Any $\mathcal{H}_l^{\mathrm{res}} \in \mathcal{S}_n$ satisfies $\|\mathcal{H}_l^{\mathrm{res}}\|_2 = 1$. Moreover, $\mathcal{S}_n$ is closed under multiplication (\textit{Proof in Appendix~\ref{app:closure}}), preventing signal amplification and gradient explosion in deep networks.
    \item \textbf{Enhanced Expressivity.} Since all doubly stochastic matrices possess a spectral norm of one, the Birkhoff polytope $\mathcal{B}_n$ (i.e. the set of doubly stochastic matrices) is contained within $\mathcal{S}_n$. By confining to $\mathcal{S}_n$, $\mathrm{s}^{2}$HC eliminates the suppression of subdominant singular values inherent to the doubly stochastic constraint. This restored spectral freedom naturally enables the model's capacity for selectively preserving or alleviating the cross-stream variations.
\end{enumerate}

\textbf{Parameterization Equivalence via Spectral Decoupling.} 
Directly parameterizing $\mathcal{H}_l^{\mathrm{res}} \in \mathcal{A}_n$ with a strict unit spectral norm is challenging. To address this, we note that the affine subspace $\mathcal{A}_n$ is a translation of the zero-marginal subspace $\mathcal{Z}_n=\{\mathcal{H}\in\mathbb{R}^{n\times n}\mid \mathcal{H}\mathbf{1}_n=\mathbf{0}_n,\mathbf{1}_n^\top\mathcal{H}=\mathbf{0}_n^\top\}$ by the uniform matrix $J=\frac{1}{n}\mathbf{1}_n\mathbf{1}_n^\top$ (\textit{Proof in Appendix~\ref{app:affine_trans}}). Thus, any target residual matrix admits a unique decomposition $\mathcal{H}_l^{\mathrm{res}} = J + \mathcal{H}_l^{\mathrm{disp}}$, where $\mathcal{H}_l^{\mathrm{disp}} \in \mathcal{Z}_n$. 

Then consider any column $\bm{x} \in \mathbb{R}^n$ of the expanded residual streams $X_l$, which can be uniquely decomposed as $\bm{x} = \bm{x}_{\parallel} + \bm{x}_{\perp},  \bm{x}_{\parallel} \in \mathrm{span}\{\mathbf{1}_n\}, \ \bm{x}_{\perp} \in \mathbf{1}_n^\perp.$ The operators $J$ and $\mathcal{H}_l^{\mathrm{disp}}$ act orthogonally on these components:
$J \bm{x}_{\perp} = \mathbf{0}_n,  \mathcal{H}_l^{\mathrm{disp}}\bm{x}_{\parallel} = \mathbf{0}_n.$
This decouples their spectral contributions and leads to Proposition~\ref{prop:spectral_bound} (\textit{Proof in Appendix~\ref{app:bound_proof}}).

\begin{proposition}[Spectral Decoupling]
\label{prop:spectral_bound}
Let $J = \frac{1}{n}\mathbf{1}_n\mathbf{1}_n^\top$. For any $\mathcal{H}_l^{\mathrm{disp}} \in \mathcal{Z}_n$, the spectral norm of the corresponding residual matrix $\mathcal{H}_l^{\mathrm{res}} = J + \mathcal{H}_l^{\mathrm{disp}}$ satisfies:
\begin{equation}
    \|\mathcal{H}_l^{\mathrm{res}}\|_2 = \max\left(\|J\|_2, \|\mathcal{H}_l^{\mathrm{disp}}\|_2\right)
\end{equation}
\end{proposition}

\begin{remark}
\text
Since $\|J\|_2 = 1$, Proposition~\ref{prop:spectral_bound} establishes that enforcing $\|\mathcal{H}_l^{\mathrm{res}}\|_2 = 1$ in the affine subspace $\mathcal{A}_n$ is equivalent to bounding $\|\mathcal{H}_l^{\mathrm{disp}}\|_2 \le 1$ in the zero-marginal subspace $\mathcal{Z}_n$. 
\end{remark}

\subsection{Spectral-Sphere-Constrained Hyper-Connections}
As established in Proposition~\ref{prop:spectral_bound}, enforcing 
$\mathcal{H}_l^{\mathrm{res}} \in \mathcal{S}_n$ reduces to constructing 
$\mathcal{H}_l^{\mathrm{disp}} \in \mathcal{Z}_n$ satisfying 
$\|\mathcal{H}_l^{\mathrm{disp}}\|_2 \le 1$.
Since $\mathcal{H}_l^{\mathrm{disp}}$ lies in the zero-marginal subspace $\mathcal{Z}_n$, its rank is at most $n-1$. 
We therefore parameterize it via the compact singular value decomposition (SVD):
\begin{equation}
\label{eq:svd}
\mathcal{H}_l^{\mathrm{disp}} = U_l \Sigma_l V_l^\top
\end{equation}
where $U_l,V_l \in \mathbb{R}^{n\times (n-1)}$ satisfy 
$U_l^\top U_l = V_l^\top V_l = I_{n-1}$, and 
$\Sigma_l \in \mathbb{R}^{(n-1)\times(n-1)}$ is a diagonal matrix containing singular values. 

To satisfy the zero-sum constraint (i.e. $\mathcal{H}_l^{\mathrm{disp}}\in\mathcal{Z}_n$), we constrain the singular vectors to lie in $\mathbf{1}_n^\perp$ via factorizing:
\begin{equation}
U_l = U_{\mathcal Z} U_l^{\mathrm{core}}, \qquad V_l = U_{\mathcal Z} V_l^{\mathrm{core}}
\end{equation}
where $U_{\mathcal Z}\in\mathbb{R}^{n\times(n-1)}$ denotes the truncated Helmert matrix, which forms an orthonormal basis of subspace $\mathbf{1}_n^\perp$, and $U_l^{\mathrm{core}},V_l^{\mathrm{core}}\in\mathbb{R}^{(n-1)\times(n-1)}$ are orthogonal matrices. Substituting these factorizations into the the compact SVD in Eq.~\ref{eq:svd} yields our proposed parameterization, which provides a complete and spectral-controlled parameterization of the zero-marginal subspace $\mathcal{Z}_n$ as in Proposition~\ref{prop:parameterization} (\textit{Proof in Appendix~\ref{app:parameterization}}).

\begin{proposition}[Complete and Spectral-Controlled Parameterization]
\label{prop:parameterization}
Let $U_{\mathcal Z} \in \mathbb{R}^{n \times (n-1)}$ be an orthonormal basis of the subspace $\mathbf{1}_n^\perp$, and let $U_l^{\mathrm{core}}, V_l^{\mathrm{core}}\in \mathbb{R}^{(n-1)\times(n-1)}$ be orthogonal matrices. The parameterization,
\begin{equation}
\mathcal{H}_l^{\mathrm{disp}} = (U_{\mathcal Z} U_l^{\mathrm{core}}) \, \Sigma_l \, (U_{\mathcal Z} V_l^{\mathrm{core}})^\top
\end{equation}
provides a complete representation of the zero-marginal subspace $\mathcal{Z}_n$. Furthermore, bounding the singular values in $\Sigma_l$ provides exact control over the spectral norm $\|\mathcal{H}_l^{\mathrm{disp}}\|_2$.
\end{proposition}

\begin{remark}
As established in Proposition~\ref{prop:parameterization}, parameterizing $\mathcal{H}_l^{\mathrm{disp}}$ satisfying $\|\mathcal{H}_l^{\mathrm{disp}}\|_2 \le1$ reduces to generating the orthogonal matrices 
$U_l^{\mathrm{core}}, V_l^{\mathrm{core}}$ and bounding $\Sigma_l$.
\end{remark}
 
Given the flattened vector $\bm{x}_l\in\mathbb{R}^{1\times nC}$ from the expanded streams $X_l$ at $l$-th layer, we dynamically generate these components:
\begin{empheq}[left=\empheqlbrace]{equation}
\begin{aligned}
&\bm{x}_l'=\text{RMSNorm}(\bm{x}_l) \\
&U_l^{\mathrm{core}} = \operatorname{Cayley}\Big(\operatorname{skew}\big( \gamma_l^U\tanh(\tau_l^U (\bm{x}_l'W_l^U) + \bm{b}_l^U) \big)\Big) \\
&V_l^{\mathrm{core}} = \operatorname{Cayley}\Big(\operatorname{skew}\big( \gamma_l^V\tanh(\tau_l^V(\bm{x}_l'W_l^V) + \bm{b}_l^V) \big)\Big) \\
&\Sigma_l = \operatorname{diag}\Big(\tanh\big(\tau_l^S(\bm{x}_l'W_l^S) + \bm{b}_l^S\big)\Big)
\end{aligned}
\end{empheq}
Here, $W_l^U, W_l^V \in \mathbb{R}^{nC \times k}$ (with $k=\frac{1}{2}(n-1)(n-2)$) and $W_l^S \in \mathbb{R}^{nC \times (n-1)}$ are learnable projection weights, while $\bm{b}_l^U, \bm{b}_l^V\in\mathbb{R}^{1\times k}$ and $\bm{b}_l^S\in\mathbb{R}^{1\times (n-1)}$ are learnable biases. The parameters $\tau_l^U, \tau_l^V, \tau_l^S$ are trainable scalar factors, and $\gamma_l^U, \gamma_l^V$ act as learnable rotation magnitude gates. The operator $\operatorname{skew}(\cdot)$ constructs a skew-symmetric matrix in $\mathbb{R}^{(n-1) \times (n-1)}$ by populating its strictly upper-triangular entries with the $k$ activated outputs and completing the lower triangle via anti-symmetry. $\operatorname{Cayley}(\cdot)$ is the Cayley transform which maps the constructed skew-symmetric matrices into orthogonal matrices $U_l^{\mathrm{core}}$ and $V_l^{\mathrm{core}}$. $\operatorname{diag}(\cdot)$ maps the output $n-1$ singular values onto the main diagonal, where the $\tanh(\cdot)$ bounds the singular values $|(\Sigma_l)_{i,i}|\le1$ for $i=1\dots n-1$, thus ensuring $\|\mathcal{H}_l^{\mathrm{disp}}\|_2 \le 1$.

Substituting the parametrized matrices and adding the translation $J$ yields the final residual matrix:
\begin{equation}
\mathcal{H}_l^{\mathrm{res}}=J+(U_{\mathcal Z} U_l^{\mathrm{core}})\Sigma_l(U_{\mathcal Z} V_l^{\mathrm{core}})^\top
\end{equation}
Figure~\ref{fig:overview} illustrates the geometric intuition of our parameterization and its associated constraint spaces.
Finally, except for the above residual matrix $\mathcal{H}_l^{\mathrm{res}}$, our $\mathrm{s}^{2}$HC keeps other structures of mHC unchanged.


%% file: section/6_experiment.tex
\section{Experiment}
\label{sec:experiment}

\textbf{Models and Datasets.} To demonstrate the generalizability of our method, we evaluate $\mathrm{s}^{2}$HC across diverse model architectures on configurations both with and without tied word embeddings. Accordingly, leveraging the LitGPT framework~\cite{litgpt-2023}, we evaluate $\mathrm{s}^{2}$HC on three distinct base models: an untied Qwen3-0.6B variant ($\sim$0.75B parameters)~\cite{qwen3technicalreport}, the default tied Gemma3-1B~\cite{gemmateam2025gemma3technicalreport} and Qwen2.5-1.5B~\cite{qwen2025qwen25technicalreport}. For all models, all standard architectural hyperparameters are strictly preserved. The sole architectural modification is the expansion of standard residual connections into corresponding hyper-connections. Models are trained from scratch on the FineWebEdu dataset~\cite{penedo2024fineweb}. Adhering to compute-optimal scaling laws~\cite{hoffmann2022training}, we allocate a $20\times$ training token budget relative to their actual parameter counts (including embedding parameters). This yields approximately 15B tokens for Qwen3, 20B tokens for Gemma3 and 30B tokens for Qwen2.5.

\textbf{Baselines.} We evaluate different constrained hyper-connection paradigms including Manifold-Constrained Hyper-Connections using Sinkhorn-Knopp iterations (mHC~\cite{xie2025mhc}), and its permutation-based variants (mHC-lite~\cite{yang2026mhc} and KromHC~\cite{zhou2026kromhc}). Aligned with their settings, the number of residual streams for all hyper-connections including our $\mathrm{s}^{2}$HC is set to $n=4$. We further include the standard single-stream residual connection (RC~\cite{he2016deep}), a reference for assessing the magnitude of the observed improvements.

\textbf{Initialization.} Following mHC/mHC-lite/KromHC, we adopt their initialization schemes so that each variant reduces to a standard residual connection (identity mapping) at initialization. $\mathrm{s}^{2}$HC is also set as identity mapping at initialization.

\textbf{Evaluation.} 
To assess the performance of these hyper-connections, we evaluate the trained models on eight diverse benchmarks, including PIQA~\cite{bisk2020piqa}, WinoGrande~\cite{sakaguchi2019winogrande}, HellaSwag~\cite{zellers2019hellaswag}, BoolQ~\cite{clark2019boolq}, ARC-E and ARC-C~\cite{clark2018think}, MMLU~\cite{hendryckstest2021}, and TruthfulQA~\cite{lin-etal-2022-truthfulqa}, using the lm-eval harness~\cite{eval-harness}. 
These benchmarks cover a broad range of capabilities, including commonsense reasoning, reading comprehension, multi-domain knowledge reasoning, and factual truthfulness.

Detailed experiment setup is provided in Appendix~\ref{app:hyperparameters}.
\begin{table}[t]
\vspace{-2em}
    \centering
    \caption{Evaluation results (accuracy $\%$) on benchmarks. Higher values indicate better performance.}
    \scalebox{0.8}{
    \renewcommand{\arraystretch}{1.1}
    \begin{tabular}{lccccccccc}
    \toprule
    \textbf{Method} & \textbf{WinoGrande} & \textbf{BoolQ} & \textbf{PIQA} & \textbf{HellaSwag} & \textbf{ARC-E} & \textbf{ARC-C} & \textbf{MMLU} & \textbf{TruthfulQA} &\textbf{Avg.} \\
    \cmidrule(lr){1-10}
    \rowcolor[gray]{0.95} \multicolumn{10}{c}{\textbf{\textit{Qwen3}}} \\
    \midrule
    RC &50.6 &58.4 &67.1 &42.0 &\underline{58.8} &\underline{29.7} &25.3 &\textbf{36.7} &\underline{46.1}\\
    mHC & \underline{53.5} & 59.8 &\underline{67.7} & \underline{43.6} & 56.6 &29.3 & \underline{25.4} & 32.9 &\underline{46.1}\\
    mHC-lite & 53.4  & 60.1 &67.1  & 43.1 & 56.4 & 28.7 & 24.3 & 33.1 &45.8 \\
    KromHC & 49.5 & \textbf{60.4} &\underline{67.7}  & 43.2 & 56.3 & 28.3 & 24.7 & 32.4 &45.3 \\
    \textbf{$\mathrm{s}^{2}$HC (Ours)} & \textbf{53.6} & \underline{60.3} &\textbf{68.3}  & \textbf{44.0} & \textbf{59.8} & \textbf{30.9} & \textbf{25.6} & \underline{34.4} &\textbf{47.1}\\
    \midrule
    \rowcolor[gray]{0.95} \multicolumn{10}{c}{\textbf{\textit{Gemma3}}} \\
    \midrule
    RC &52.5 &\textbf{62.1} &69.3 &47.0 &65.3 &35.0 &24.3 &35.7 &48.9\\
    mHC &51.7 &60.7 &69.6 &\textbf{48.2} &67.0 &33.9 &25.2 &\textbf{38.9} &49.4\\
    mHC-lite &53.2 &\underline{61.3} &70.2 &\underline{48.1} &67.3 &\underline{36.3} &\underline{25.4} &37.3 &49.9\\
    KromHC &\underline{53.3} &59.4 &\underline{70.6} &48.0 &\textbf{67.6} &\textbf{36.6} &\textbf{25.8} &\underline{38.4} &\underline{50.0}\\
    \textbf{$\mathrm{s}^{2}$HC (Ours)} &\textbf{55.3} &60.6 &\textbf{70.7} &\textbf{48.2} &\underline{67.5} &36.0 &25.3 &37.9 &\textbf{50.2}\\
    \midrule
    \rowcolor[gray]{0.95} \multicolumn{10}{c}{\textbf{\textit{Qwen2.5}}} \\
    \midrule
    RC &51.4 &60.9 &70.9 &49.9 &64.2 &34.5 &\underline{25.4} &34.9 &49.0\\
    mHC &54.5 &62.1 &\textbf{71.6} &\textbf{51.7} &61.2 &35.2 &24.5 &35.3 &49.5 \\
    mHC-lite &53.4 &\textbf{62.4} &70.7 &51.3 &63.3 &34.9 &25.1 &\underline{35.6} &\underline{49.6} \\
    KromHC &\underline{55.1} &\underline{62.3} &\underline{71.3} &50.4 &\underline{64.3} &\underline{35.3} &\textbf{25.7} &31.5 &49.5 \\
    \textbf{$\mathrm{s}^{2}$HC (Ours)} &\textbf{55.2} &62.2 &\textbf{71.6} &\underline{51.5} &\textbf{66.5} &\textbf{37.0} &25.1 &\textbf{35.7} &\textbf{50.7} \\
    \bottomrule
    \end{tabular}
    }
    \label{tab:benchmark}
    \vspace{-1em}
\end{table}

\subsection{Performance and Expressivity Analysis}
\label{sec:performance}

\textbf{Performance.} As shown in Table~\ref{tab:benchmark}, $\mathrm{s}^{2}$HC demonstrates cross-model consistency, achieving the highest average accuracy on all base models. Specifically on Qwen3, $\mathrm{s}^{2}$HC ranks first on six out of eight benchmarks, with notable gains over the strongest hyper-connection baseline on reasoning (+3.2\% on ARC-E, +1.6\% on ARC-C). This trend also holds on Qwen2.5 (+2.2\% on ARC-E, +1.7\% on ARC-C), while the baselines themselves remain unstable across the three models. Against the standard residual connection (RC), the strongest doubly stochastic variant improves by +0.0\%, +1.1\% and +0.6\% average points on Qwen3, Gemma3 and Qwen2.5, whereas $\mathrm{s}^{2}$HC adds a further +1.0\%, +0.2\% and +1.1\% on top of them. The gain $\mathrm{s}^{2}$HC contributes over existing constrained hyper-connections is thus comparable to what those methods obtain over RC.

\textbf{Expressivity.} The right panel of Figure~\ref{fig:expressivity} traces the spectrum of $\mathcal{H}_l^{\mathrm{res}}$ within the subspace orthogonal to the mean component (i.e., the singular values of 
$\mathcal{H}_l^{\mathrm{disp}}$). The two leading subdominant singular values ($\sigma_1,\sigma_2$) remain near 1 across layers, while the smallest one ($\sigma_3$) fluctuates below 1, reaching as low as 0.77. This contrast reveals that, instead of rigidly preserving all directions, the model selectively retains the primary ones while adaptively attenuating others, exercising exactly the spectral freedom that $\mathrm{s}^{2}$HC restores. The left panel shows how this manifests in the mean pairwise cosine similarity among streams. Unlike the baselines, whose trajectories stay close to the identity mapping above 0.4, the similarity under $\mathrm{s}^{2}$HC departs markedly, dropping near 0 across early and middle layers and recovering in deeper ones. Cross-stream variations are thus introduced where they are useful and discarded where they are not. This aligns with our observation (Appendix~\ref{app:expressivity}) that $\mathrm{s}^{2}$HC leverages expressive, non-degenerate residual matrices for active stream feature mixing.

\begin{figure}[t]
\vspace{-1em}
    \centering
    \begin{minipage}[c]{0.37\textwidth}
        \centering
        \includegraphics[width=\linewidth]{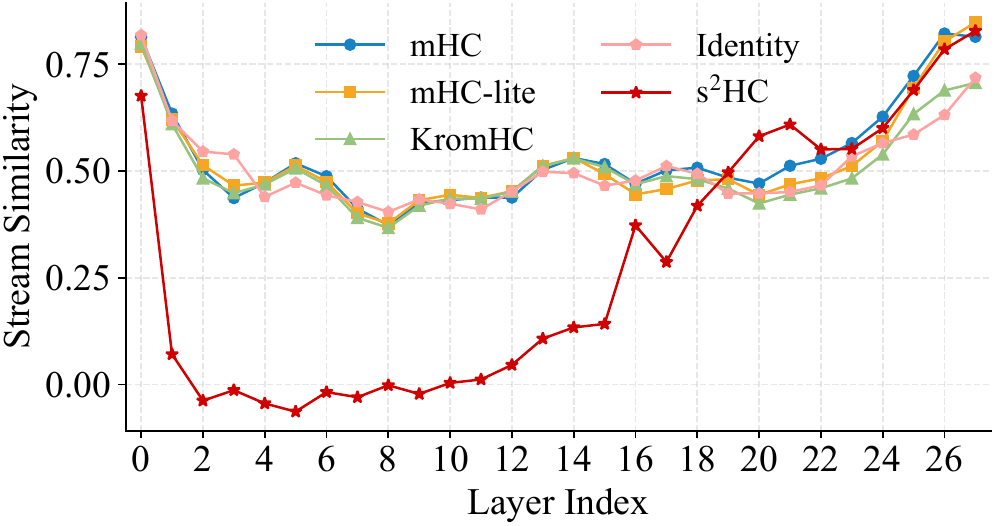}
    \end{minipage}
    \begin{minipage}[c]{0.35\textwidth}
    \centering
        \includegraphics[width=\linewidth]{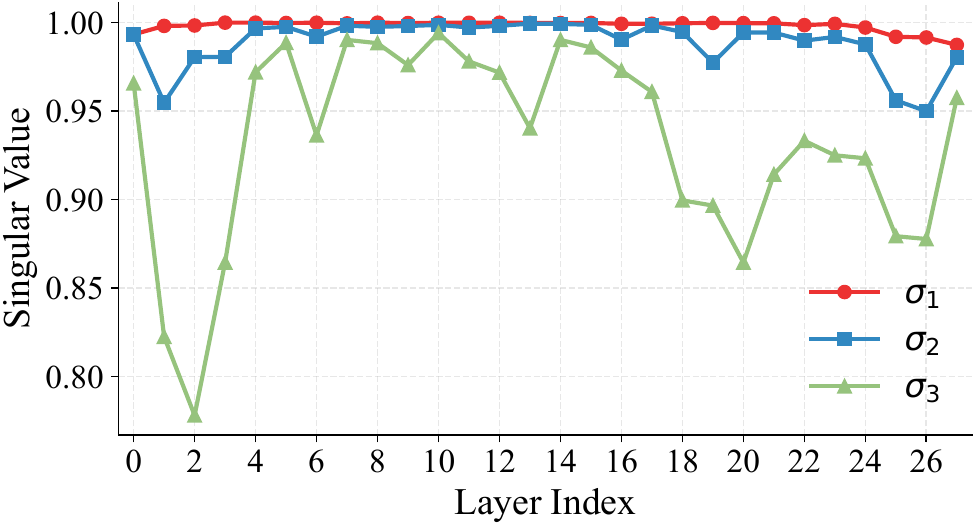}
    \end{minipage}
    \caption{Expressivity analysis of $\mathrm{s}^{2}$HC on Qwen3. \textbf{Left:} Mean pairwise cosine similarity among residual streams after being mixed by $\mathcal{H}_l^{\mathrm{res}}$ across model layers. Identity denotes the baseline where $\mathcal{H}_l^{\mathrm{res}}$ is fixed as an identity matrix. \textbf{Right:} the spectrum of $\mathcal{H}_l^{\mathrm{res}}$ within the subspace orthogonal to the mean component. More results on Gemma3 and Qwen2.5 can be found in Appendix~\ref{app:more_results}.}
    \label{fig:expressivity}
    \vspace{-1.5em}
\end{figure}

\textbf{Ablation Study.} We verify the effectiveness of $\mathrm{s}^{2}$HC's dynamic cross-stream interactions, and further show that $\mathrm{s}^{2}$HC outperforms an orthogonal-constrained baseline, confirming the benefit of $\mathrm{s}^{2}$HC's flexible singular value control. Details are in Appendix~\ref{app:ablation}.

\subsection{Stability Analysis}
\begin{wrapfigure}{r}{0.36\linewidth}
    \centering
    \includegraphics[width=\linewidth]{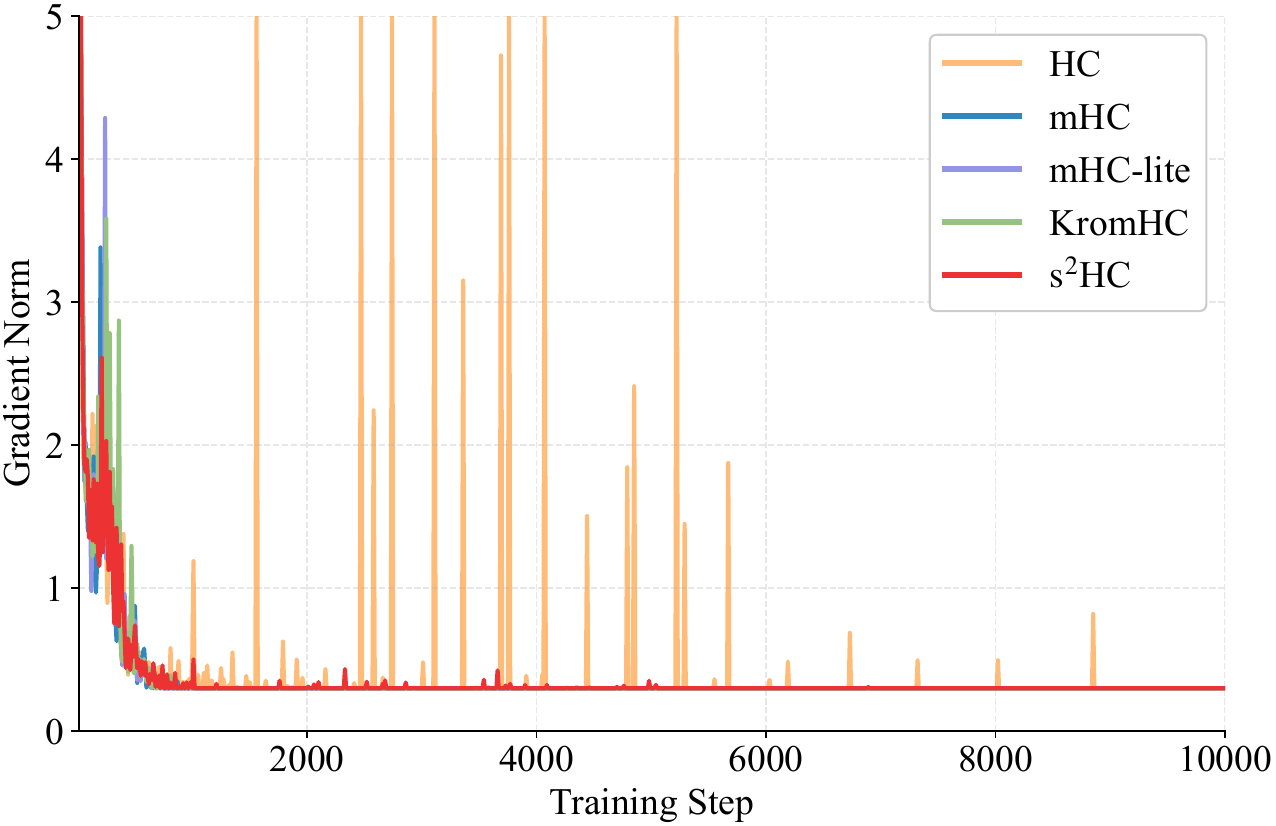}
    \caption{Gradient norm dynamics during training for the Qwen3 model on FineWebEdu. HC denotes the unconstrained Hyper-Connections~\cite{zhu2024hyper}. The unconstrained HC exhibits exploding gradients, clamped at 5 for visualization. Constrained hyper-connection paradigms show stable gradient trajectories.} 
    \label{fig:grad_norm}
\end{wrapfigure}
\textbf{Training Stability.} Figure~\ref{fig:grad_norm} shows the gradient norms during training. The unconstrained Hyper-Connections (HC)~\cite{zhu2024hyper} is unstable, and its gradient fluctuates severely during training. The constrained hyper-connections ($\mathrm{s}^{2}$HC, mHC, mHC-lite, and KromHC) stabilize the optimization. As shown, all constrained variants successfully suppress the explosive gradient spikes of HC, rapidly converging to a consistently low and stable gradient norm. However, while mHC, mHC-lite, and KromHC achieve this gradient stability at the cost of collapsing into the identity mapping, $\mathrm{s}^{2}$HC maintains optimization stability without passively sacrificing expressive, non-degenerate feature mixing.

\textbf{Hyper-Connection Stability.} We evaluate the layer-wise and propagation stability of residual matrices for mHC, mHC-lite, KromHC,  and $\mathrm{s}^{2}$HC. Results are shown in Figure~\ref{fig:stability}. As illustrated in the left panel, mHC fails to strictly satisfy the doubly stochastic constraint despite 20 Sinkhorn-Knopp iterations. The column sums of its layer-wise residual matrices $\mathcal{H}_l^\mathrm{res}$ deviate from 1.0, with a peak reaching 2.0. These column-sum deviations accumulate across layers as illustrated in the middle panel: the composite mapping $\prod_{l=0}^{27}\mathcal{H}_{27-l}^\mathrm{res}$ across 28 layers for mHC exhibits a pronounced shift of its column sums away from 1.0, with accumulated outliers spiking to 3.6. In contrast, $\mathrm{s}^{2}$HC, mHC-lite and KromHC maintain unit column sums for both the layer-wise residual matrices and the global composite mapping, thereby preserving the mean component of residual streams without signal diminishment or amplification.

Furthermore, the right panel tracks the spectral norm of the cumulative composite mapping $\prod_{l=0}^{L-1}\mathcal{H}_{L-1-l}^\mathrm{res}$ across the first $L$ layers. It reveals that mHC fails to bound this cumulative norm, which increases as model depth grows. Conversely, $\mathrm{s}^{2}$HC, mHC-lite, and KromHC constrain the cumulative spectral norm around 1.0, effectively preventing depth-induced gradient amplification. 

\begin{figure}[t]
\vspace{-1.5em}
    \centering
    \begin{minipage}[c]{0.22\textwidth}
        \centering
        \includegraphics[width=\linewidth]{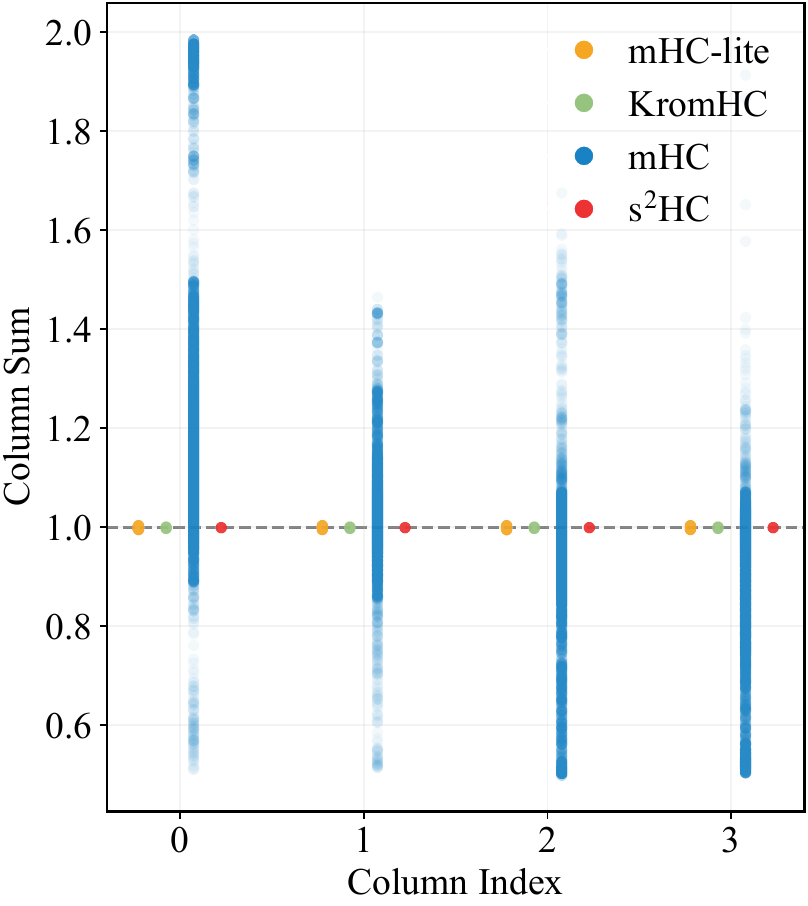}
    \end{minipage}
    \begin{minipage}[c]{0.22\textwidth}
    \centering
        \includegraphics[width=\linewidth]{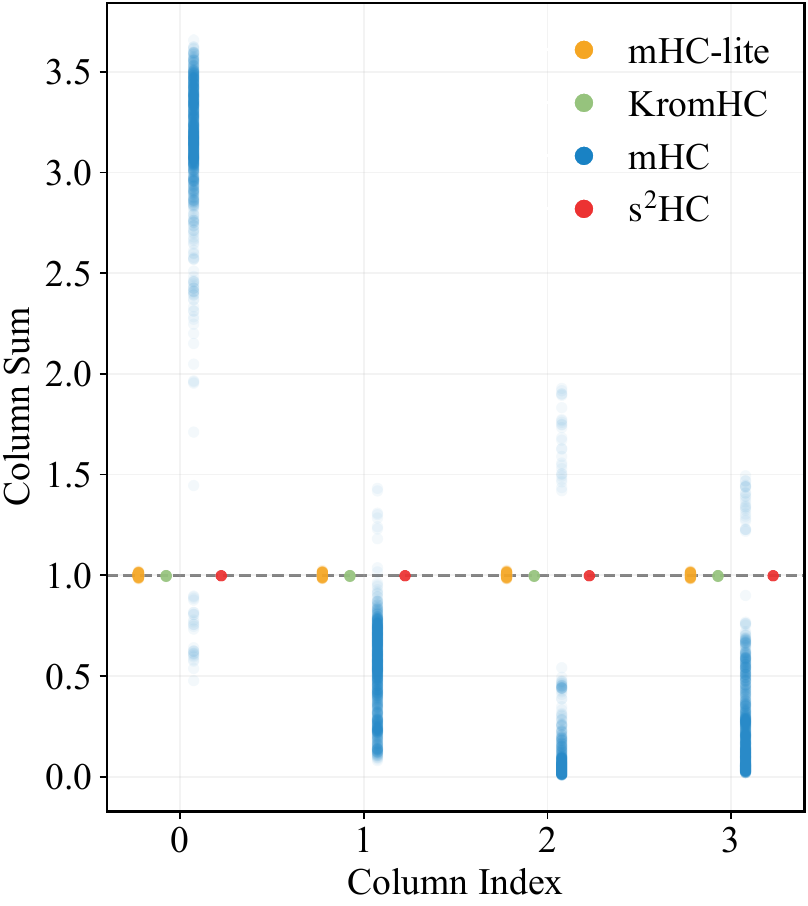}
    \end{minipage}
    \begin{minipage}[c]{0.30\textwidth}
        \centering
        \includegraphics[width=\linewidth]{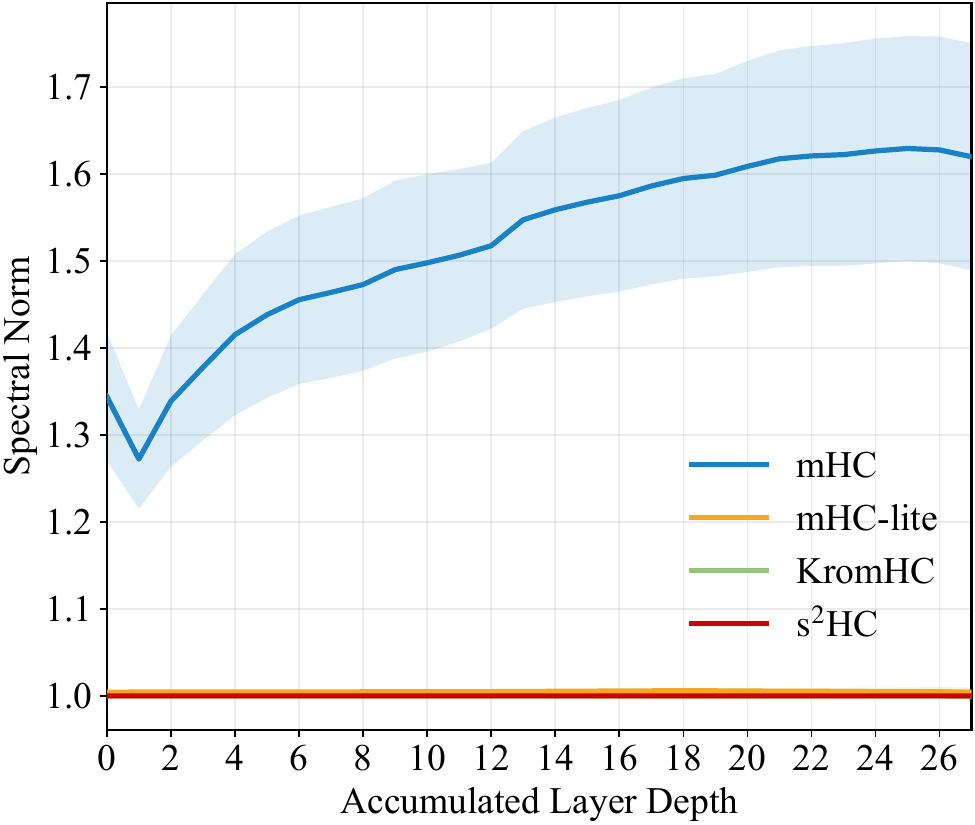}
    \end{minipage}
    \caption{Empirical validation of the stability of constrained residual matrices for mHC, mHC-lite, and $\mathrm{s}^{2}$HC. Statistics are computed from the trained Qwen3 model with 4 residual streams on the 1024 samples. \textbf{Left:} Column-sum distribution of layer-wise residual matrices $\mathcal{H}_l^\mathrm{res}$ for each layer. \textbf{Middle:} Column-sum distribution of the full-model composite mapping $\prod_{l=0}^{27}\mathcal{H}_{27-l}^\mathrm{res}$. \textbf{Right:} Spectral norm of the cumulative composite mapping $\|\prod_{l=0}^{L-1}\mathcal{H}_{L-1-l}^\mathrm{res}\|_2$ across the first $L$ layers. Lines denote means across samples while shaded regions indicate standard deviations. More results on Gemma3 and Qwen2.5 are provided in Appendix~\ref{app:more_results}.}
    \label{fig:stability}
    \vspace{-1em}
\end{figure}

\subsection{Efficiency and Parameterization Overhead}
\begin{figure}
    \centering
    \includegraphics[width=0.5\linewidth]{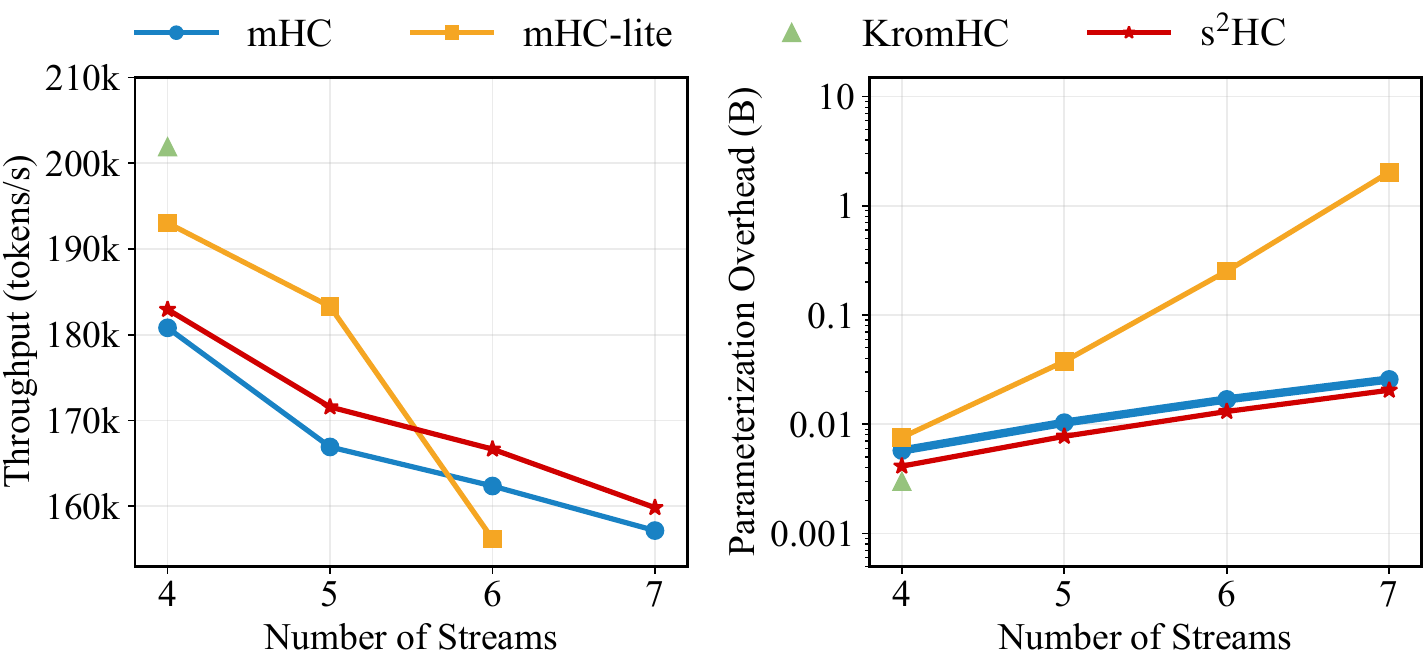}
    \caption{Training throughput and parameterization overhead across varying numbers of streams $n$. The mHC result is based on our PyTorch re-implementation and may underestimate the throughput of the specialized-kernel implementation in \cite{xie2025mhc}. KromHC is plotted as discrete points at $n=4$ due to its restriction to powers of two under the default setting. 
    \textbf{Left:} Training throughput evaluated on a 4$\times$H200 setup with a fixed global batch size. The data point for mHC-lite at $n=7$ is omitted, as its factorial parameter explosion exceeds the GPU memory budget, resulting in an Out-Of-Memory failure. 
    \textbf{Right:} Parameterization overhead (params) of the hyper-connections, excluding base model weights.}
    \label{fig:efficiency}
    \vspace{-1em}
\end{figure}

We evaluate the end-to-end training throughput (tokens/s) and parameterization overhead (params) of mHC, mHC-lite, KromHC, and $\mathrm{s}^{2}$HC based on the Qwen3 model. All experiments are conducted on a 4$\times$H200 GPU setup with a fixed sequence length of 2048 tokens and a constant 512 global batch size (micro-batch 8).

As shown in Figure~\ref{fig:efficiency}, $\mathrm{s}^{2}$HC scales stably with compact parameterizations. For $\mathrm{s}^{2}$HC, the projection weights ($W_l^U, W_l^V, W_l^S$) have a combined parameter complexity corresponding to $\mathbb{R}^{nC\times (n-1)^2}$. Similar to mHC ( $W_l^{\mathrm{res}}\in \mathbb{R}^{nC\times n^2}$), they grow cubically with the number of streams $n$. In contrast, the auxiliary parameters of mHC-lite grow factorially with $n$. This difference in scaling explains why $\mathrm{s}^{2}$HC and mHC maintain stable throughput, whereas the throughput of mHC-lite degrades rapidly beyond $n=5$. At $n=7$, the auxiliary parameters of mHC-lite reach 2B, approximately 3.3 times the size of the base model, exceeding our GPU memory and resulting in an Out-of-Memory failure while mHC and $\mathrm{s}^{2}$HC introduce only 26M and 20M auxiliary parameters respectively. 

While KromHC alleviates the factorial complexity via Kronecker products to achieve competitive throughput, its efficiency critically depends on the integer factorization of $n$. When $n$ includes large prime factors and cannot be decomposed into small integers, the method effectively reverts to intractable factorial complexity~\cite{zhou2026kromhc}. Consequently, in its default settings, KromHC restricts the number of streams to powers of two (e.g., $n = 4, 8$) to preserve the efficiency of the parameterization. This limitation prevents it from efficiently supporting arbitrary stream widths. In contrast, $\mathrm{s}^{2}$HC avoids both the factorial scaling of mHC-lite and the structural constraints of KromHC, providing a scalable and flexible solution for arbitrary residual stream counts.

%% file: section/7_conclusion.tex
\section{Conclusion}
In this paper, we introduce Spectral-Sphere-Constrained Hyper-Connections ($\mathrm{s}^{2}$HC) to eliminate factorial parameterization overhead, and alleviate the identity degeneration and expressivity bottleneck inherent to prior doubly stochastic constraint hyper-connections. By reformulating the constraint to an affine-constrained spectral norm sphere, $\mathrm{s}^{2}$HC overcomes these limitations. We hope that the expressivity, stability, and parameterization efficiency of $\mathrm{s}^{2}$HC can potentially illuminate new pathways to contribute to ongoing research efforts to diversify connectivity patterns beyond standard residual architectures. Further discussion regarding the broader impacts and limitations of $\mathrm{s}^{2}$HC is provided in Appendix~\ref{app:discussion}.

%% file: section/appendix.tex
\newpage
\appendix

\section{Expressivity Analysis}
\label{app:expressivity}
\begin{figure}[t]
    \includegraphics[width=\linewidth]{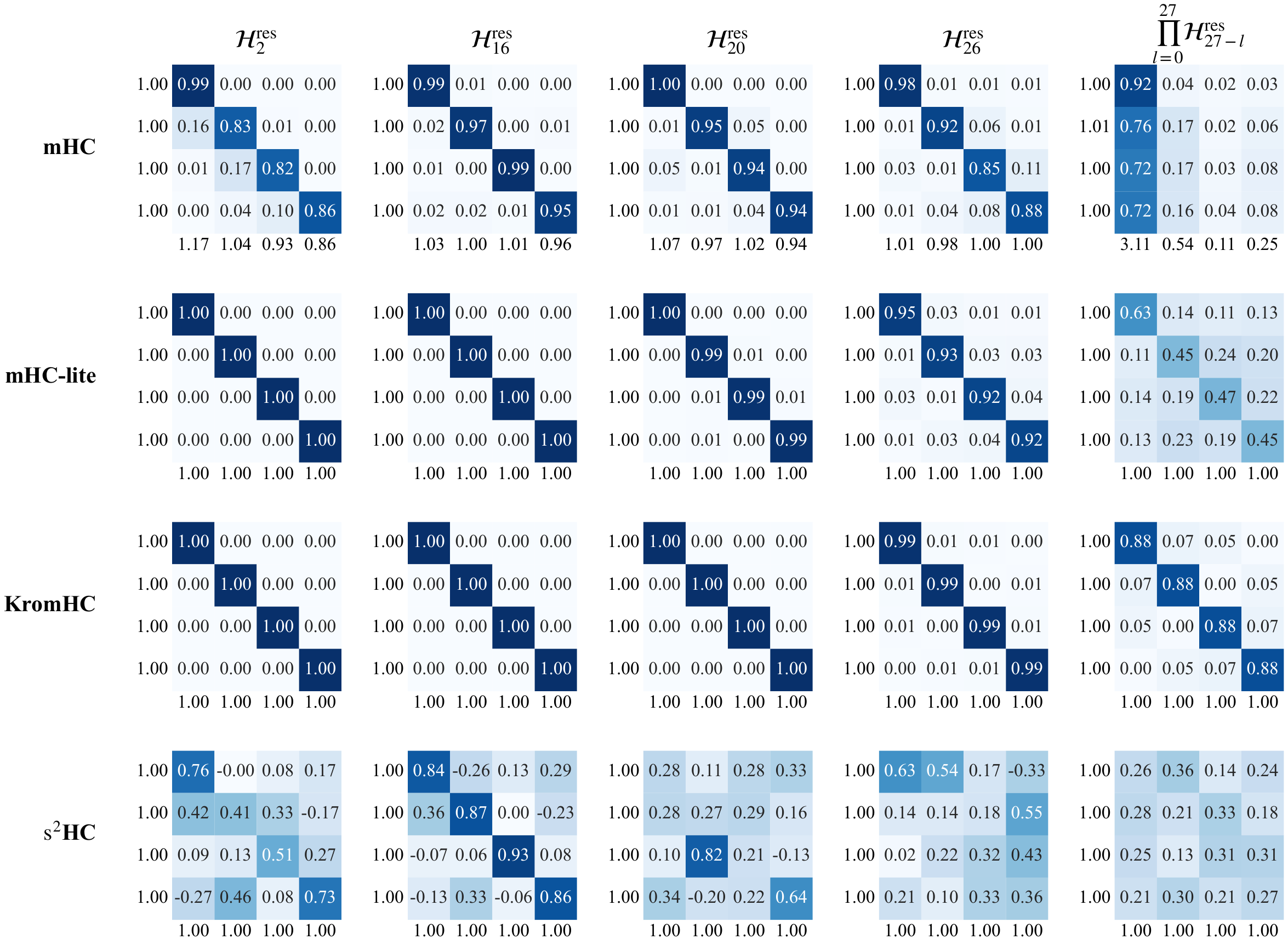}
    \caption{Visualization of layer-wise residual matrices and composite mappings on Qwen3. The figure displays single-layer residual matrices $\mathcal{H}_l^{\mathrm{res}}$ at selected depths ($l \in \{2, 16, 20, 26\}$), alongside their end-to-end composite mappings ($\prod_{l=0}^{27} \mathcal{H}_{27-l}^{\mathrm{res}}$) across mHC, mHC-lite, KromHC, and $\mathrm{s}^{2}$HC (Ours). Each matrix is computed by averaging over all tokens within a sequence. The numerical annotations along the vertical and horizontal axes denote the row sums and column sums, respectively.}
    \label{fig:matrix_vis}
\end{figure}

\begin{figure}[h]
    \centering
    \includegraphics[width=0.9\linewidth]{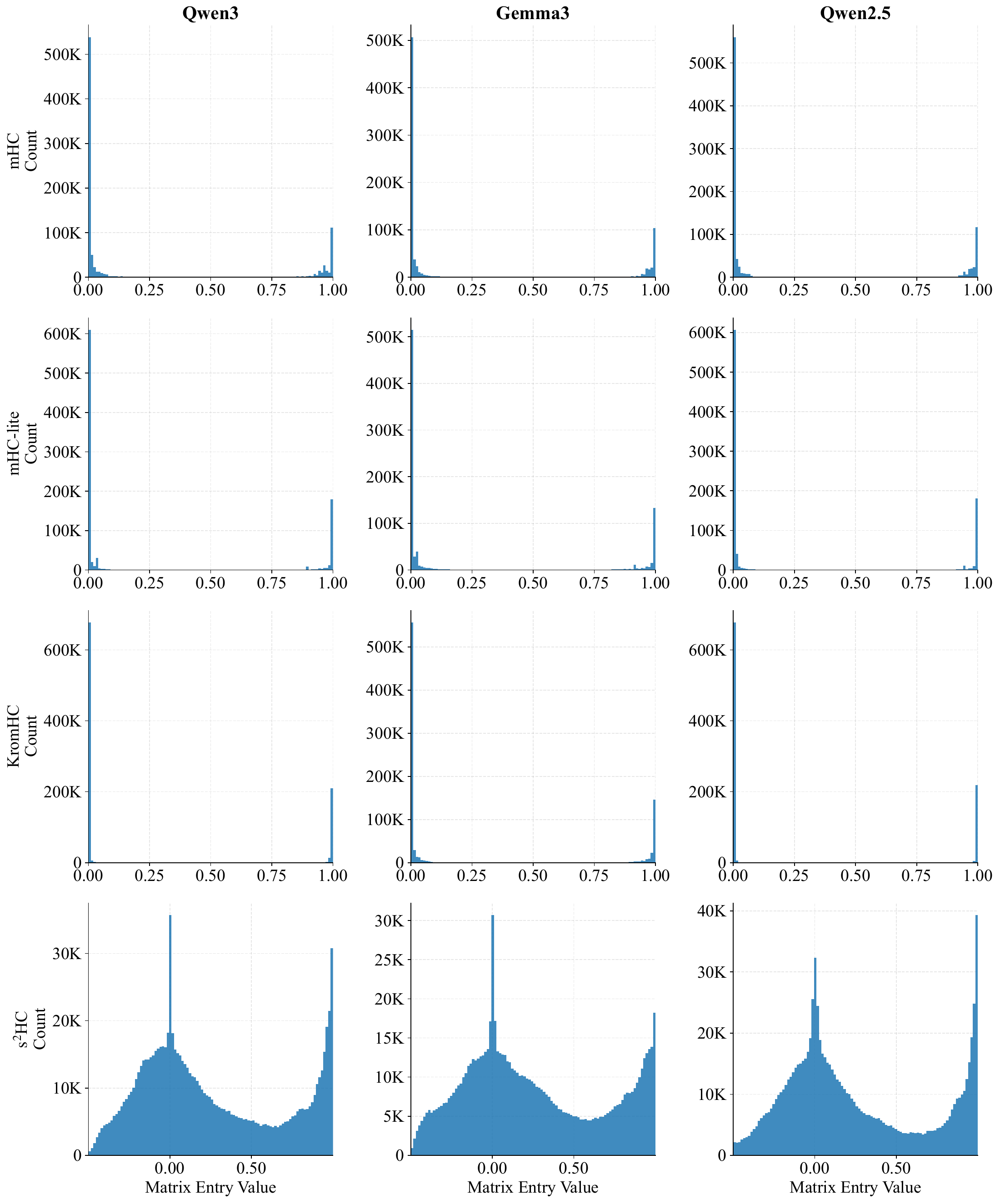}
    \caption{Distribution of layer-wise residual matrix entries in Qwen3 across mHC, mHC-lite, KromHC and our proposed $\mathrm{s}^{2}$HC.}
    \label{fig:entry_distribution}
\end{figure}

\textbf{Residual Matrix Dynamics.}
We provide the visualizations (Figure~\ref{fig:matrix_vis},~\ref{fig:more_heat_map}) and entry distributions (Figure~\ref{fig:entry_distribution}) of the learned layer-wise residual matrices $\mathcal{H}_l^{\mathrm{res}}$ based on the trained Qwen3, Gemma3 and Qwen2.5 model. For all doubly stochastic constrained hyper-connections (mHC, mHC-lite, and KromHC), matrix entries sharply concentrate around 0.00 and 1.00, producing sparse and near-identity layer-wise residual matrices. These observations are consistent with the identity degeneration discussed in \S~\ref{sec:observation}. In contrast, $\mathrm{s}^{2}$HC exhibits a continuous entry distribution, indicating that it leverages non-degenerated residual matrices for active feature interactions. Importantly, this increased expressivity does not compromise stability. As shown in Figure~\ref{fig:matrix_vis}, $\mathrm{s}^{2}$HC maintains exact row and column sums of 1.00 across all depths, whereas mHC accumulates normalization errors, resulting in composite column sums ranging from 0.11 to 3.11.

\section{Ablation Study}
\label{app:ablation}
In this section, we conduct two ablation studies to verify the effectiveness of $\mathrm{s}^{2}$HC. In each study, only the studied component is modified, while all other settings follow those of the main experiments in \S~\ref{sec:experiment}.

\textbf{Comparison with Identity Mapping Baseline.} To verify the effectiveness of $\mathrm{s}^{2}$HC's dynamic cross-stream interactions, we conduct an ablation study using an unmixed multi-stream baseline. In this setup, we remove the learnable residual matrix $\mathcal{H}_l^{\mathrm{res}}$ (while retaining $\mathcal{H}_l^{\mathrm{pre}},\mathcal{H}_l^{\mathrm{post}}$), which is equivalent to fixing the connection to a static multi-stream identity mapping. The results are presented in Table~\ref{tab:ablation_res}.

As shown, enabling dynamic stream mixing yields consistent improvements across all three models and most benchmarks, with the static baseline matching $\mathrm{s}^{2}$HC on only a few individual tasks. This consistency indicates that $\mathrm{s}^{2}$HC makes effective use of its learnable residual matrices under the spectral-sphere constraint, capturing beneficial cross-stream interactions beyond a static identity mapping.
\begin{table}[t]
\centering
\caption{Comparison with Identity Mapping Baseline.}
\label{tab:ablation_res}
\small
\setlength{\tabcolsep}{3.5pt}
\begin{tabular}{c|ccccccccc}
\toprule
$\mathcal{H}_l^{\mathrm{res}}$ &\textbf{WinoGrande} & \textbf{BoolQ} & \textbf{PIQA} & \textbf{HellaSwag} & \textbf{ARC-E} & \textbf{ARC-C} & \textbf{MMLU} & \textbf{TruthfulQA} \\
\midrule
\rowcolor[gray]{0.95} \multicolumn{9}{c}{\textbf{\textit{Qwen3}}} \\
\midrule
\ding{55} &53.3 &\textbf{60.5} &67.4 &43.2 &58.9 &\textbf{31.1} &24.0 &34.2 \\
\checkmark & \textbf{53.6} &60.3 &\textbf{68.3}  & \textbf{44.0} & \textbf{59.8} &30.9 & \textbf{25.6} & \textbf{34.4} \\
\midrule
\rowcolor[gray]{0.95} \multicolumn{9}{c}{\textbf{\textit{Gemma3}}} \\
\midrule
\ding{55} &52.9 &59.6 &70.3 &48.0 &67.0 &35.4 &24.3 &36.8 \\
\checkmark &\textbf{55.3} &\textbf{60.6} &\textbf{70.7} &\textbf{48.2} &\textbf{67.5} &\textbf{36.0} &\textbf{25.3} &\textbf{37.9} \\
\midrule
\rowcolor[gray]{0.95} \multicolumn{9}{c}{\textbf{\textit{Qwen2.5}}} \\
\midrule
\ding{55} &\textbf{55.2} &60.3 &\textbf{71.6} &50.2 &65.3 &35.1 &\textbf{25.2} &33.1 \\
\checkmark &\textbf{55.2} &\textbf{62.2} &\textbf{71.6} &\textbf{51.5} &\textbf{66.5} &\textbf{37.0} &25.1 &\textbf{35.7} \\
\bottomrule
\end{tabular}
\end{table}

\textbf{Comparison with Orthogonal-Constrained Baseline.} An intuitive way to prevent the suppression of the subdominant spectrum while retaining the idea of constrained hyper-connections is to restrict the residual matrix to be orthogonal. An orthogonal matrix fixes all singular values to one, which inherently bounds the spectral norm and avoids the collapse of subdominant singular values. To examine whether this design suffices, we dynamically generate orthogonal residual matrices from the normalized residual input $\bm{x}_l'\in \mathbb{R}^{1\times nC}$:
\begin{equation}
    \mathcal{H}_l^{\mathrm{res}} = \operatorname{Cayley}\Big(\operatorname{skew}\big( \gamma_l^O\tanh(\tau_l^O (\bm{x}_l'W_l^O) + \bm{b}_l^O) \big)\Big)
\end{equation}
Here, $W_l^O \in \mathbb{R}^{nC\times q}$ (with $q=\tfrac{1}{2}n(n-1)$) is a learnable projection weight and $\bm{b}_l^O \in \mathbb{R}^{1\times q}$ is a learnable bias, while $\tau_l^O$ is a trainable scalar factor and $\gamma_l^O$ is a learnable rotation magnitude gate. The operator $\operatorname{skew}(\cdot)$ populates the strictly upper-triangular entries of an $n\times n$ skew-symmetric matrix with the $q$ activated outputs, and $\operatorname{Cayley}(\cdot)$ maps it to an orthogonal matrix. Except for constraining  $\mathcal{H}_l^{\mathrm{res}}$ to be orthogonal, we keep all other components unchanged and identical to $\mathrm{s}^{2}$HC.

The results are reported in Table~\ref{tab:ablation_orth}. As shown, $\mathrm{s}^{2}$HC consistently outperforms orthogonal-constrained baseline on most benchmarks across all three models. This indicates that beyond preventing spectral suppression, the ability of $\mathrm{s}^{2}$HC to selectively attenuate singular values, rather than rigidly fixing them to one, contributes to its effectiveness. This substantiates our observation on the spectrum of the learned residual matrices of $\mathrm{s}^{2}$HC as in Figure~\ref{fig:expressivity}.
\begin{table}[t]
\centering
\caption{Comparison with Orthogonal-Constrained Baseline.}
\label{tab:ablation_orth}
\small
\setlength{\tabcolsep}{3.5pt}
\begin{tabular}{c|ccccccccc}
\toprule
& \textbf{WinoGrande} & \textbf{BoolQ} & \textbf{PIQA} & \textbf{HellaSwag} & \textbf{ARC-E} & \textbf{ARC-C} & \textbf{MMLU} & \textbf{TruthfulQA} \\
\midrule
\rowcolor[gray]{0.95} \multicolumn{9}{c}{\textbf{\textit{Qwen3}}} \\
\midrule
Orthogonal &52.8 &\textbf{60.5} &67.9 &43.9 &\textbf{59.8} &30.4 &24.3 &\textbf{34.7} \\
$\mathrm{s}^{2}$HC (Ours) &\textbf{53.6} &60.3 &\textbf{68.3}  & \textbf{44.0} & \textbf{59.8} &\textbf{30.9} & \textbf{25.6} &34.4 \\
\midrule
\rowcolor[gray]{0.95} \multicolumn{9}{c}{\textbf{\textit{Gemma3}}} \\
\midrule
Orthogonal &53.1 &59.5 &\textbf{71.5} &\textbf{48.5} &65.7 &33.8 &25.2 &\textbf{38.6} \\
$\mathrm{s}^{2}$HC (Ours) &\textbf{55.3} &\textbf{60.6} &70.7 &48.2 &\textbf{67.5} &\textbf{36.0} &\textbf{25.3} &37.9 \\
\midrule
\rowcolor[gray]{0.95} \multicolumn{9}{c}{\textbf{\textit{Qwen2.5}}} \\
\midrule
Orthogonal &52.9 &60.1 &70.2 &50.1 &\textbf{66.5} &33.7 &\textbf{26.2} &33.9\\
$\mathrm{s}^{2}$HC (Ours) &\textbf{55.2} &\textbf{62.2} &\textbf{71.6} &\textbf{51.5} &\textbf{66.5} &\textbf{37.0} &25.1 &\textbf{35.7} \\
\bottomrule
\end{tabular}
\end{table}

Moreover, the orthogonal constraint discards the mean-preservation property. This property underlies existing constrained hyper-connections~\cite{xie2025mhc,yang2026mhc,zhou2026kromhc} and is important for the training stability of deep models. Figure~\ref{fig:grad_norm_orth} shows the gradient norms of the orthogonal-constrained baseline and $\mathrm{s}^{2}$HC during training on three models.

As shown, the orthogonal-constrained baseline exhibits persistently large and oscillating gradient norms throughout training. In contrast, $\mathrm{s}^{2}$HC rapidly converges to a consistently low and stable gradient norm. This confirms that, without the mean-preservation guarantee, orthogonal residual matrices may induce unstable optimization, underscoring the necessity of the affine constraint in $\mathrm{s}^{2}$HC.
\begin{figure}[t]
    \centering
    \begin{minipage}[c]{0.29\textwidth}
        \centering
        Qwen3
    \end{minipage}
    \begin{minipage}[c]{0.29\textwidth}
    \centering
        Gemma3
    \end{minipage}
    \begin{minipage}[c]{0.29\textwidth}
        \centering
        Qwen2.5
    \end{minipage}
    \begin{minipage}[c]{0.30\textwidth}
        \centering
        \includegraphics[width=\linewidth]{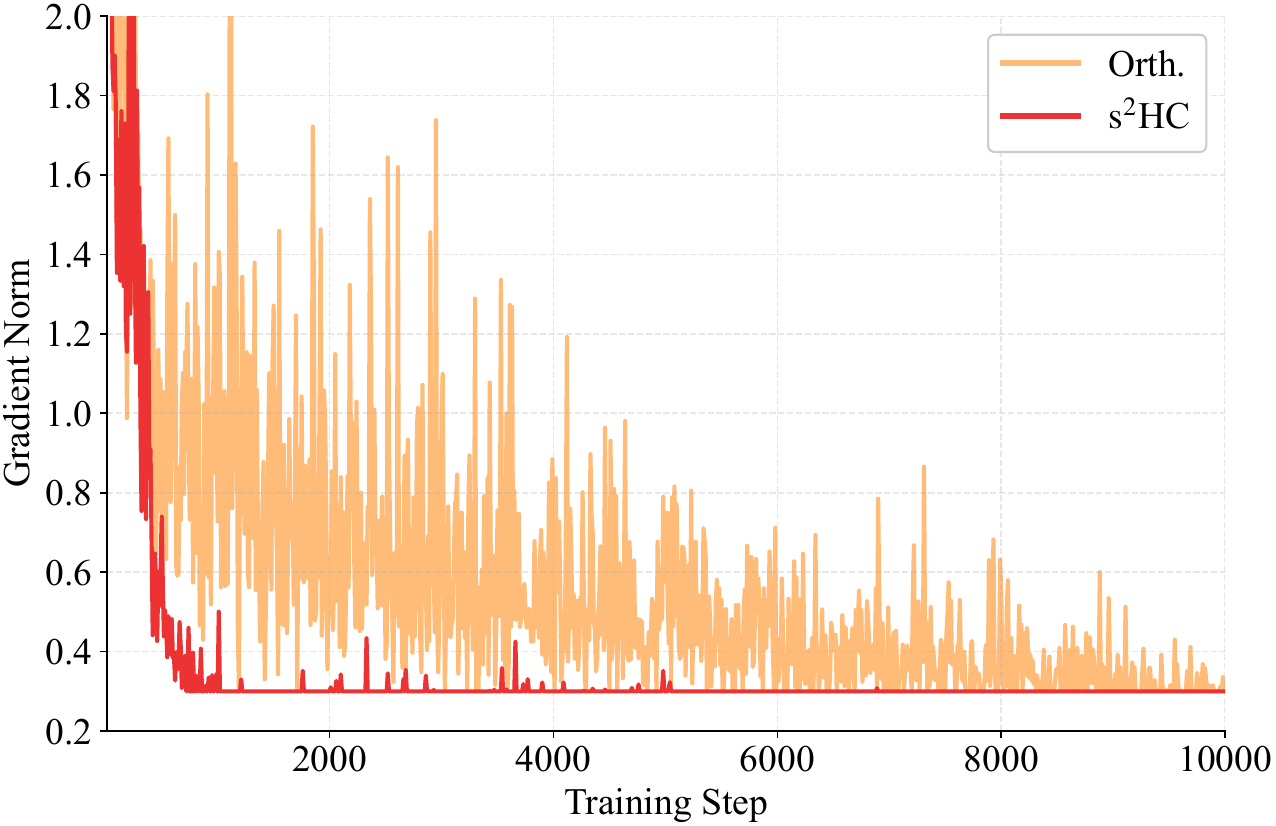}
    \end{minipage}
    \begin{minipage}[c]{0.30\textwidth}
    \centering
        \includegraphics[width=\linewidth]{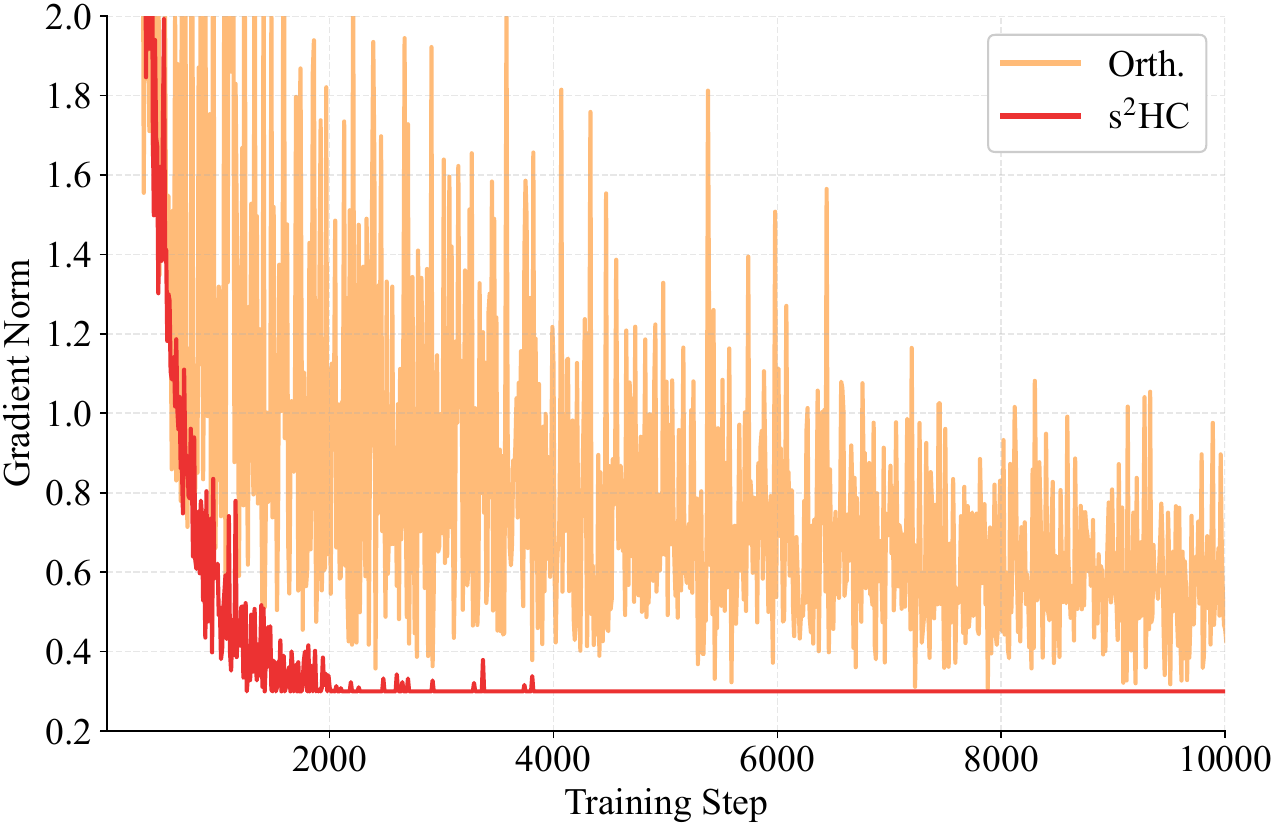}
    \end{minipage}
    \begin{minipage}[c]{0.30\textwidth}
        \centering
        \includegraphics[width=\linewidth]{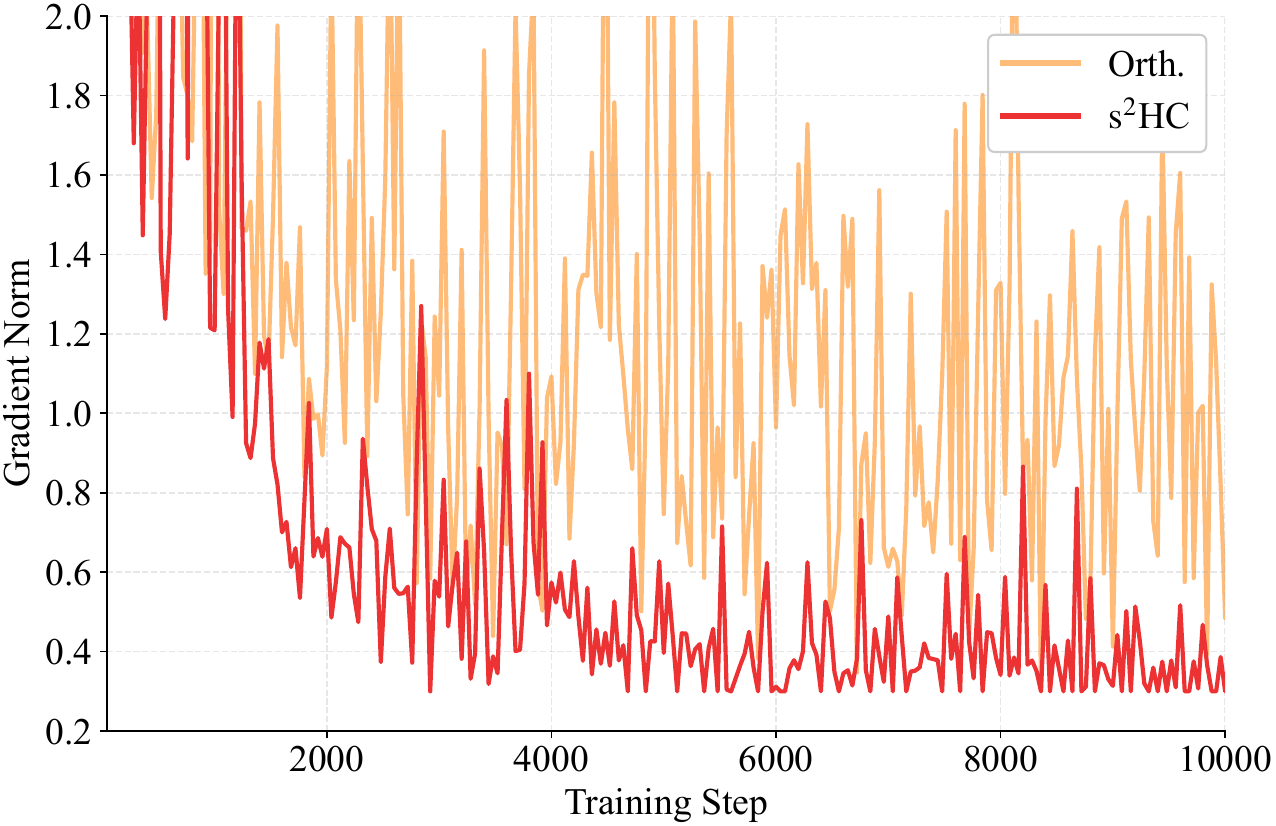}
    \end{minipage}
    \caption{Gradient norm dynamics during training on three models. Orth. denotes the orthogonal-constrained baseline. The results are clamped at 2 for visualization.}
    \label{fig:grad_norm_orth}
\end{figure}

\section{Discussion}
\label{app:discussion}
\textbf{Broader Impacts.}
In this paper, our analysis uncovers that while existing doubly stochastic constrained hyper-connections stabilize training by restricting residual matrices to a doubly stochastic manifold, this constraint causes the learned matrices to degenerate toward their identity initialization, running counter to the original intention of hyper-connections to facilitate active cross-stream feature interactions.
Our proposed framework, Spectral-Sphere-Constrained Hyper-Connections ($\mathrm{s}^{2}$HC), addresses these limitations, enabling active feature mixing, while ensuring training stability through an exact, closed-form spectral constraint. Furthermore, the proposed parameterization enforces this bound without introducing the factorial parameter overhead inherent in prior exact-constraint methods. Consequently, $\mathrm{s}^{2}$HC renders hyper-connection designs generalizable, contributing to the ongoing research efforts to diversify connectivity patterns beyond standard residual architectures. By reconciling expressivity with parameterization efficiency, $\mathrm{s}^{2}$HC offers the community a reference paradigm for deploying the hyper-connection architecture.

\textbf{Limitations.}
While our $\mathrm{s}^{2}$HC enhances model expressivity and eliminates the parameterization inefficiency of existing constrained hyper-connections, it inherits the memory access overhead intrinsic to the multi-stream residual architecture of existing hyper-connections~\cite{xie2025mhc,yang2026mhc,zhou2026kromhc,zhu2024hyper}. This increased demand for memory bandwidth can restrict actual training throughput at scale. However, this problem may be mitigated through system-level kernel optimizations, as demonstrated by mHC~\cite{xie2025mhc}. In addition, our experiments are conducted at the 0.6B–1.5B scale models under compute-optimal token budgets, reflecting the compute available to us. The validation of $\mathrm{s}^{2}$HC at substantially larger model and token scales is still remaining as the future work.

\section{Proofs}
\subsection{Multiplicative Closure}
\label{app:closure}
\begin{lemma}[Multiplicative Closure of $\mathcal{S}_n$]
The affine-constrained spectral norm sphere $\mathcal{S}_n = \{\mathcal{H} \in \mathcal{A}_n \mid \|\mathcal{H}\|_2 = 1\}$ is closed under matrix multiplication.
\end{lemma}

\begin{proof}
For any $\mathcal H_1,\mathcal H_2\in\mathcal S_n$, the affine constraints are preserved under multiplication:
\begin{equation}
    \begin{aligned}
        (\mathcal H_1\mathcal H_2)\mathbf1_n
&=
\mathcal H_1(\mathcal H_2\mathbf1_n)
=
\mathcal H_1\mathbf1_n
=
\mathbf1_n, \\
\mathbf1_n^\top(\mathcal H_1\mathcal H_2)
&=
(\mathbf1_n^\top\mathcal H_1)\mathcal H_2
=
\mathbf1_n^\top\mathcal H_2
=
\mathbf1_n^\top .
    \end{aligned}
\end{equation}

Hence $\mathcal H_1\mathcal H_2\in\mathcal A_n$.

For the spectral norm, submultiplicativity gives
\begin{equation}
    \|\mathcal H_1\mathcal H_2\|_2
\le
\|\mathcal H_1\|_2\|\mathcal H_2\|_2
=1.
\end{equation}

On the other hand, since $\mathcal H_1\mathcal H_2\mathbf1_n=\mathbf1_n$, we have
\begin{equation}
    \|\mathcal H_1\mathcal H_2\|_2
=
\sup_{\bm x\ne0}
\frac{\|\mathcal H_1\mathcal H_2\bm x\|_2}{\|\bm x\|_2}
\ge
\frac{\|\mathcal H_1\mathcal H_2\mathbf1_n\|_2}{\|\mathbf1_n\|_2}
=
1.
\end{equation}

Therefore
\begin{equation}
    1\le \|\mathcal H_1\mathcal H_2\|_2 \le 1,
\end{equation}
which implies $\|\mathcal H_1\mathcal H_2\|_2=1$ and hence $\mathcal H_1\mathcal H_2\in\mathcal S_n$.
\end{proof}

\subsection{Affine Subspace Translation}
\label{app:affine_trans}
\begin{lemma}[Affine Subspace Translation]
\label{lemma:translation}
Let $J = \frac{1}{n}\mathbf{1}_n\mathbf{1}_n^\top$ be the uniform matrix, and let $\mathcal{Z}_n = \{\mathcal{H} \in \mathbb{R}^{n \times n} \mid \mathcal{H}\mathbf{1}_n = \mathbf{0}_n, \mathbf{1}_n^\top\mathcal{H} = \mathbf{0}_n^\top\}$ be the zero-marginal subspace. The affine subspace $\mathcal{A}_n = \{\mathcal{H} \in \mathbb{R}^{n \times n} \mid \mathcal{H}\mathbf{1}_n = \mathbf{1}_n, \mathbf{1}_n^\top\mathcal{H} = \mathbf{1}_n^\top\}$ is a translation of $\mathcal{Z}_n$ by $J$, i.e., $\mathcal{A}_n = J + \mathcal{Z}_n$.
\end{lemma}
\begin{proof}

Let \(J=\frac1n\mathbf1_n\mathbf1_n^\top\), which satisfies
\begin{equation}
    J\mathbf1_n=\mathbf1_n, \qquad \mathbf1_n^\top J=\mathbf1_n^\top .
\end{equation}

For any matrix \(\mathcal H\),
\begin{equation}
\mathcal H\in\mathcal A_n
\;\Longleftrightarrow\;
(\mathcal H-J)\mathbf1_n=\mathbf0_n,\;
\mathbf1_n^\top(\mathcal H-J)=\mathbf0_n^\top
\;\Longleftrightarrow\;
\mathcal H-J\in\mathcal Z_n .
\end{equation}

Therefore,
\begin{equation}
\mathcal A_n = J + \mathcal Z_n .
\end{equation}
\end{proof}

\subsection{Spectral Decoupling}
\label{app:bound_proof}
\begin{proof}[Proof of Proposition~\ref{prop:spectral_bound}]
Let $\bm{x} \in \mathbb{R}^n$, which can be orthogonally decomposed $\bm{x} = \bm{x}_{\parallel} + \bm{x}_{\perp}$, where $\bm{x}_{\parallel} \in \mathrm{span}\{\mathbf{1}_n\}$ and $\bm{x}_{\perp} \in \mathbf{1}_n^\perp$. Thus, $\|\bm{x}\|_2^2 = \|\bm{x}_{\parallel}\|_2^2 + \|\bm{x}_{\perp}\|_2^2$.

Since $J = \frac{1}{n}\mathbf{1}_n\mathbf{1}_n^\top$ and $\mathcal{H}_l^{\mathrm{disp}} \in \mathcal{Z}_n$, we have:
\begin{equation}
    J\bm{x}_{\perp} = \mathbf{0}_n, \quad \mathcal{H}_l^{\mathrm{disp}}\bm{x}_{\parallel} = \mathbf{0}_n, \quad \text{and} \quad \mathbf{1}_n^\top \mathcal{H}_l^{\mathrm{disp}}\bm{x}_{\perp} = 0.
\end{equation}

Applying $\mathcal{H}_l^{\mathrm{res}} = J + \mathcal{H}_l^{\mathrm{disp}}$ (Lemma~\ref{lemma:translation}) to $\bm{x}$, the action decouples:
\begin{equation}
\label{eq:22}
    \mathcal{H}_l^{\mathrm{res}}\bm{x} = (J + \mathcal{H}_l^{\mathrm{disp}})(\bm{x}_{\parallel} + \bm{x}_{\perp}) = J\bm{x}_{\parallel} + \mathcal{H}_l^{\mathrm{disp}}\bm{x}_{\perp}.
\end{equation}

Since $J\bm{x}_{\parallel} \in \mathrm{span}\{\mathbf{1}_n\}$ and $\mathcal{H}_l^{\mathrm{disp}}\bm{x}_{\perp} \in \mathbf{1}_n^\perp$, the output terms are mutually orthogonal. Consequently:
\begin{equation}
\begin{aligned}
    \|\mathcal{H}_l^{\mathrm{res}}\bm{x}\|_2^2 &= \|J\bm{x}_{\parallel}\|_2^2 + \|\mathcal{H}_l^{\mathrm{disp}}\bm{x}_{\perp}\|_2^2 \\
    &\le \|J\|_2^2 \|\bm{x}_{\parallel}\|_2^2 + \|\mathcal{H}_l^{\mathrm{disp}}\|_2^2 \|\bm{x}_{\perp}\|_2^2 \\
    &\le \max\left(\|J\|_2^2, \|\mathcal{H}_l^{\mathrm{disp}}\|_2^2\right) \left(\|\bm{x}_{\parallel}\|_2^2 + \|\bm{x}_{\perp}\|_2^2\right) \\
    &= \max\left(\|J\|_2^2, \|\mathcal{H}_l^{\mathrm{disp}}\|_2^2\right) \|\bm{x}\|_2^2.
\end{aligned}
\end{equation}

Dividing by $\|\bm{x}\|_2^2$ and taking the supremum over $\bm{x} \ne \mathbf{0}_n$ explicitly establishes the upper bound:
\begin{equation}
    \|\mathcal{H}_l^{\mathrm{res}}\|_2 = \sup_{\bm{x} \ne \mathbf{0}} \frac{\|\mathcal{H}_l^{\mathrm{res}}\bm{x}\|_2}{\|\bm{x}\|_2} \le \max\left(\|J\|_2, \|\mathcal{H}_l^{\mathrm{disp}}\|_2\right).
\end{equation}

To establish the lower bound, we evaluate the quotient over specific subspaces.

For any non-zero vector $\bm{u} \in \mathrm{span}\{\mathbf{1}_n\}$, we have $\mathcal{H}_l^{\mathrm{disp}}\bm{u} = \mathbf{0}_n$, which implies $\mathcal{H}_l^{\mathrm{res}}\bm{u} = J\bm{u}.$
Therefore,
\begin{equation}
    \|\mathcal{H}_l^{\mathrm{res}}\|_2 = \sup_{\bm{x}\neq\mathbf{0}} \frac{\|\mathcal{H}_l^{\mathrm{res}}\bm{x}\|_2}{\|\bm{x}\|_2} \ge \sup_{\bm{u}\in\mathrm{span}\{\mathbf{1}_n\},\,\bm{u}\neq\mathbf{0}} \frac{\|J\bm{u}\|_2}{\|\bm{u}\|_2} = \|J\|_2.
\end{equation}

Similarly, for any non-zero vector $\bm{v} \in \mathbf{1}_n^\perp$, we have $J\bm{v} = \mathbf{0}$, which implies $\mathcal{H}_l^{\mathrm{res}}\bm{v} = \mathcal{H}_l^{\mathrm{disp}}\bm{v}.$
Thus,
\begin{equation}
    \|\mathcal{H}_l^{\mathrm{res}}\|_2 = \sup_{\bm{x}\neq\mathbf{0}} \frac{\|\mathcal{H}_l^{\mathrm{res}}\bm{x}\|_2}{\|\bm{x}\|_2} \ge \sup_{\bm{v}\in\mathbf{1}_n^\perp,\,\bm{v}\neq\mathbf{0}} \frac{\|\mathcal{H}_l^{\mathrm{disp}}\bm{v}\|_2}{\|\bm{v}\|_2} = \|\mathcal{H}_l^{\mathrm{disp}}\|_2.
\end{equation}

Combining these conditions yields $\|\mathcal{H}_l^{\mathrm{res}}\|_2 \ge \max\left(\|J\|_2, \|\mathcal{H}_l^{\mathrm{disp}}\|_2\right)$. Since both the upper and lower bounds hold, exact equality is established:
\begin{equation}
    \|\mathcal{H}_l^{\mathrm{res}}\|_2 = \max\left(\|J\|_2, \|\mathcal{H}_l^{\mathrm{disp}}\|_2\right).
\end{equation}
\end{proof}

\subsection{Complete and Spectral-Controlled Parameterization}
\label{app:parameterization}

\begin{proof}[Proof of Proposition~\ref{prop:parameterization}]
\textit{Completeness.} For any $\mathcal{H}_l^{\mathrm{disp}} \in \mathcal{Z}_n$ with rank $r \le n-1$, consider its compact SVD:
\begin{equation}
\mathcal{H}_l^{\mathrm{disp}} = \tilde{U}_l \, \tilde{\Sigma}_l \, \tilde{V}_l^\top,
\end{equation}
where $\tilde{U}_l, \tilde{V}_l \in \mathbb R^{n\times r}$ satisfy $\tilde{U}_l^\top \tilde{U}_l = \tilde{V}_l^\top \tilde{V}_l = I_{r}$, and $\tilde{\Sigma}_l \in \mathbb R^{r \times r}$ is a diagonal matrix containing non-zero singular values.

From $\mathcal{H}_l^{\mathrm{disp}} \mathbf{1}_n = \mathbf{0}_n$ and $\mathbf{1}_n^\top \mathcal{H}_l^{\mathrm{disp}} = \mathbf{0}_n^\top$, we obtain
\begin{equation}
\tilde{U}_l \, \tilde{\Sigma}_l \, (\tilde{V}_l^\top \mathbf{1}_n) = \mathbf{0}_n, 
\quad
(\mathbf{1}_n^\top \tilde{U}_l) \, \tilde{\Sigma}_l \, \tilde{V}_l^\top = \mathbf{0}_n^\top.
\end{equation}
Since $\tilde{U}_l$ and $\tilde{V}_l$ have full column rank and $\tilde{\Sigma}_l$ is invertible, this implies
\begin{equation}
\tilde{V}_l^\top \mathbf{1}_n = \mathbf{0}_r, \quad \mathbf{1}_n^\top \tilde{U}_l = \mathbf{0}_r^\top,
\end{equation}
hence the column spaces of $\tilde{U}_l$ and $\tilde{V}_l$ lie in $\mathbf{1}_n^\perp$.

To align with the parameterization size $n-1$, since $\mathbf 1_n^\perp$ is an $(n-1)$-dimensional subspace, we can find $(n-1-r)$ orthonormal vectors in $\mathbf 1_n^\perp$ to expand $\tilde{U}_l$ and $\tilde{V}_l$ into $U_l, V_l \in \mathbb R^{n\times (n-1)}$, such that 
\begin{equation}
\begin{aligned}
    &\mathrm{col}(U_l),~\mathrm{col}(V_l) \subset \mathbf 1_n^\perp, \\
    &U_l^\top U_l=V_l^\top V_l=I_{n-1}.
\end{aligned}
\end{equation}
By padding $\tilde{\Sigma}_l$ with zeros to form an $\mathbb{R}^{(n-1)\times (n-1)}$ diagonal matrix $\Sigma_l$, we have equivalently:
\begin{equation}
\mathcal{H}_l^{\mathrm{disp}} = U_l \, \Sigma_l \, V_l^\top.
\end{equation}
Since $\mathrm{col}(U_l), \mathrm{col}(V_l) \subset \mathbf{1}_n^\perp$ and $U_{\mathcal Z}$ is an orthonormal basis of $\mathbf{1}_n^\perp$, we have
\begin{equation}
U_l = U_{\mathcal Z} \, U_{\mathcal Z}^\top U_l, \quad
V_l = U_{\mathcal Z} \, U_{\mathcal Z}^\top V_l.
\end{equation}
Define $U_l^{\mathrm{core}} = U_{\mathcal Z}^\top U_l$ and $V_l^{\mathrm{core}} = U_{\mathcal Z}^\top V_l$. Then
\begin{equation}
\begin{aligned}
(U_l^{\mathrm{core}})^\top U_l^{\mathrm{core}}
&=U_l^\top(U_{\mathcal Z}U_{\mathcal Z}^\top U_l)
=U_l^\top U_l
=I_{n-1}, \\
(V_l^{\mathrm{core}})^\top V_l^{\mathrm{core}}
&=V_l^\top(U_{\mathcal Z}U_{\mathcal Z}^\top V_l)
=V_l^\top V_l
=I_{n-1}.
\end{aligned}
\end{equation}
Therefore any $\mathcal{H}_l^{\mathrm{disp}}\in\mathcal Z_n$ can be written as
\begin{equation}
\mathcal{H}_l^{\mathrm{disp}}=(U_{\mathcal Z}U_l^{\mathrm{core}}) \Sigma_l (U_{\mathcal Z}V_l^{\mathrm{core}})^\top.
\end{equation}

\textit{Spectral norm preservation.} Given the parameterization $U_l=U_{\mathcal Z}U_l^{\mathrm{core}}$ and $V_l=U_{\mathcal Z}V_l^{\mathrm{core}}$ with orthogonal matrices $U_l^{\mathrm{core}}, V_l^{\mathrm{core}}$, we have:
\begin{equation}
U_l^\top U_l
=(U_l^{\mathrm{core}})^\top(U_{\mathcal Z}^\top U_{\mathcal Z})U_l^{\mathrm{core}}
=I_{n-1},
\end{equation}
\begin{equation}
V_l^\top V_l
=(V_l^{\mathrm{core}})^\top(U_{\mathcal Z}^\top U_{\mathcal Z})V_l^{\mathrm{core}}
=I_{n-1}.
\end{equation}
Thus both $U_l$ and $V_l$ have exact orthonormal columns. Then the spectral norm satisfies:
\begin{equation}
\|\mathcal H_l^{\mathrm{disp}}\|_2
=\|U_l\Sigma_lV_l^\top\|_2
=\|\Sigma_l\|_2
=\max_i|(\Sigma_l)_{i,i}|.
\end{equation}
Hence bounding $\Sigma_l$ directly controls the spectral norm of $\mathcal H_l^{\mathrm{disp}}$.
\end{proof}


\section{Experiment Setup}
\label{app:hyperparameters}

\textbf{Initialization.} In all cases, $W_l^{\mathrm{res}}$, $W_l^{\mathrm{pre}}$, and $W_l^{\mathrm{post}}$ are initialized to zero; $\alpha_l^{\mathrm{res}}$, $\alpha_l^{\mathrm{pre}}$, and $\alpha_l^{\mathrm{post}}$ are set to 0.01; and $\bm{b}_l^{\mathrm{pre}}$ and $\bm{b}_l^{\mathrm{post}}$ are initialized to $-1$ except for a single entry set to $1$.
The residual branch differs across variants. In mHC, off-diagonal entries are set to $-8$ and diagonal entries to $0$, yielding an identity-like matrix after Sinkhorn normalization. For mHC-lite and KromHC, $\bm{b}_l^{\mathrm{res}}$ is set to $-8$ for all entries except the entry corresponding to the identity matrix, which is set to $0$, so that after the softmax operation the weights concentrate on the identity matrix. 

For $\mathrm{s}^{2}$HC initialization, projection weights $W_l^U, W_l^V, W_l^S$ are zero-initialized. We set the biases $\bm{b}_l^U, \bm{b}_l^V$ to zero, so that $U_l^{\mathrm{core}},V_l^{\mathrm{core}}$ reduce to identity after the Cayley transform. Simultaneously, we initialize $\bm{b}_l^S$ to 4, saturating the $\tanh$ activation to yield an identity matrix for $\Sigma_l$. Finally, the rotation magnitude gates $\gamma_l^U, \gamma_l^V$ are set to 1, and the scaling factors $\tau_l^U, \tau_l^V, \tau_l^S$ are set to 0.01.

\textbf{Training Setup.}
Our implementation builds upon Qwen3-0.6B~\cite{qwen3technicalreport}, Gemma3-1B~\cite{gemmateam2025gemma3technicalreport} and Qwen2.5-1.5B~\cite{qwen2025qwen25technicalreport}. We set the max sequence length to 2048 and expand the standard residual connections into corresponding hyper-connections with all other hyperparameters of the architecture at their default values. Models are trained from scratch using the AdamW optimizer with a cosine learning rate schedule and linear warmup. We employ mixed-precision training (bfloat16) and apply gradient clipping throughout training.

All experiments are conducted on 4 NVIDIA H200 GPUs using PyTorch DistributedDataParallel (DDP). We fix the global batch size across all methods. The hyperparameters of training are summarized in Table~\ref{tab:training_hyperparameters}.
\begin{table}[h]
\centering
\caption{Training hyperparameters.}
\begin{tabular}{lccc}
\toprule
\textbf{Setting} &\textbf{Qwen3} &\textbf{Gemma3} &\textbf{Qwen2.5}\\
\midrule
Micro-batch size per GPU & 8 & 8 &4 \\
Global batch size & 512 & 512 &1024 \\
Sequence length & 2048 & 2048 &2048 \\
Training tokens & 15B & 20B &30B \\
Warmup iterations & 1000 & 1000 & 1000 \\
Weight decay & 0.1 & 0.1 &0.1 \\
$\beta_1$ & 0.9 & 0.9 &0.9 \\
$\beta_2$ & 0.95 & 0.95 &0.95 \\
Gradient clipping norm & 1.0 & 1.0 &1.0 \\
Initial learning rate & $6\times10^{-4}$ & $4\times10^{-4}$ & $4\times10^{-4}$ \\
Minimum learning rate & $6\times10^{-5}$ & $4\times10^{-5}$ & $4\times10^{-5}$ \\
Tie word embeddings &False &True &True \\
\bottomrule
\end{tabular}
\label{tab:training_hyperparameters}
\end{table}

\textbf{Evaluation Setup.} Aligned with the Hugging Face Open LLM Leaderboard v1~\footnote{https://huggingface.co/docs/leaderboards/open\_llm\_leaderboard/archive}, we utilize 25-shot for ARC-E and ARC-C, 10-shot for HellaSwag, 5-shot for Winogrande, and 0-shot for TruthfulQA. For PIQA and BoolQ, we adopt the widely established 0-shot configuration in the foundational model literature~\cite{touvron2023llama,brown2020language}. We also adapt MMLU to a 0-shot configuration to prevent unintended truncation caused by exceeding our maximum sequence length. The configurations are summarized in Table~\ref{tab:n_shots}.
\begin{table}[h]
    \centering
    \caption{Configurations of number of shots for each evaluation benchmark.}
    \scalebox{0.88}{
    \begin{tabular}{lcccccccc}
    \toprule
         \textbf{Benchmarks} & PIQA & BoolQ & MMLU & TruthfulQA & WinoGrande & HellaSwag & ARC-E & ARC-C \\
         \midrule
         \textbf{\# Shots} & 0 & 0 & 0 & 0 & 5 & 10 & 25 & 25 \\
    \bottomrule
    \end{tabular}
    }
    \label{tab:n_shots}
\end{table}

\section{More Results}
\label{app:more_results}

\begin{figure}[h]
    \centering
    \begin{minipage}[c]{0.04\textwidth}
        \centering
        \rotatebox{90}{Gemma3}
    \end{minipage}%
    \begin{minipage}{0.9\linewidth}
        \centering
        \includegraphics[width=\linewidth]{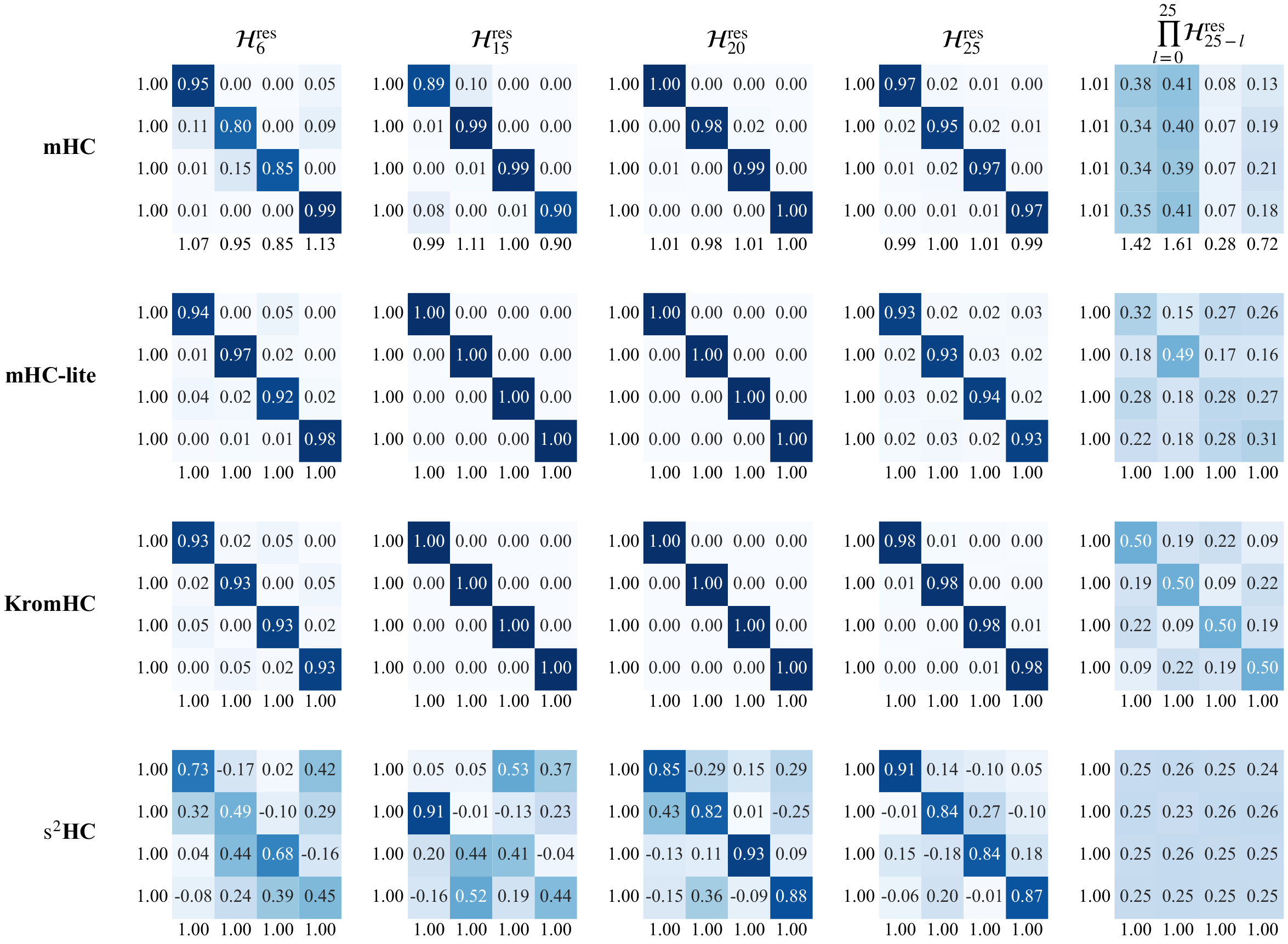}
    \end{minipage}
    
    \begin{minipage}[c]{0.04\textwidth}
        \centering
        \rotatebox{90}{Qwen2.5}
    \end{minipage}%
    \begin{minipage}{0.9\linewidth}
        \centering
        \includegraphics[width=\linewidth]{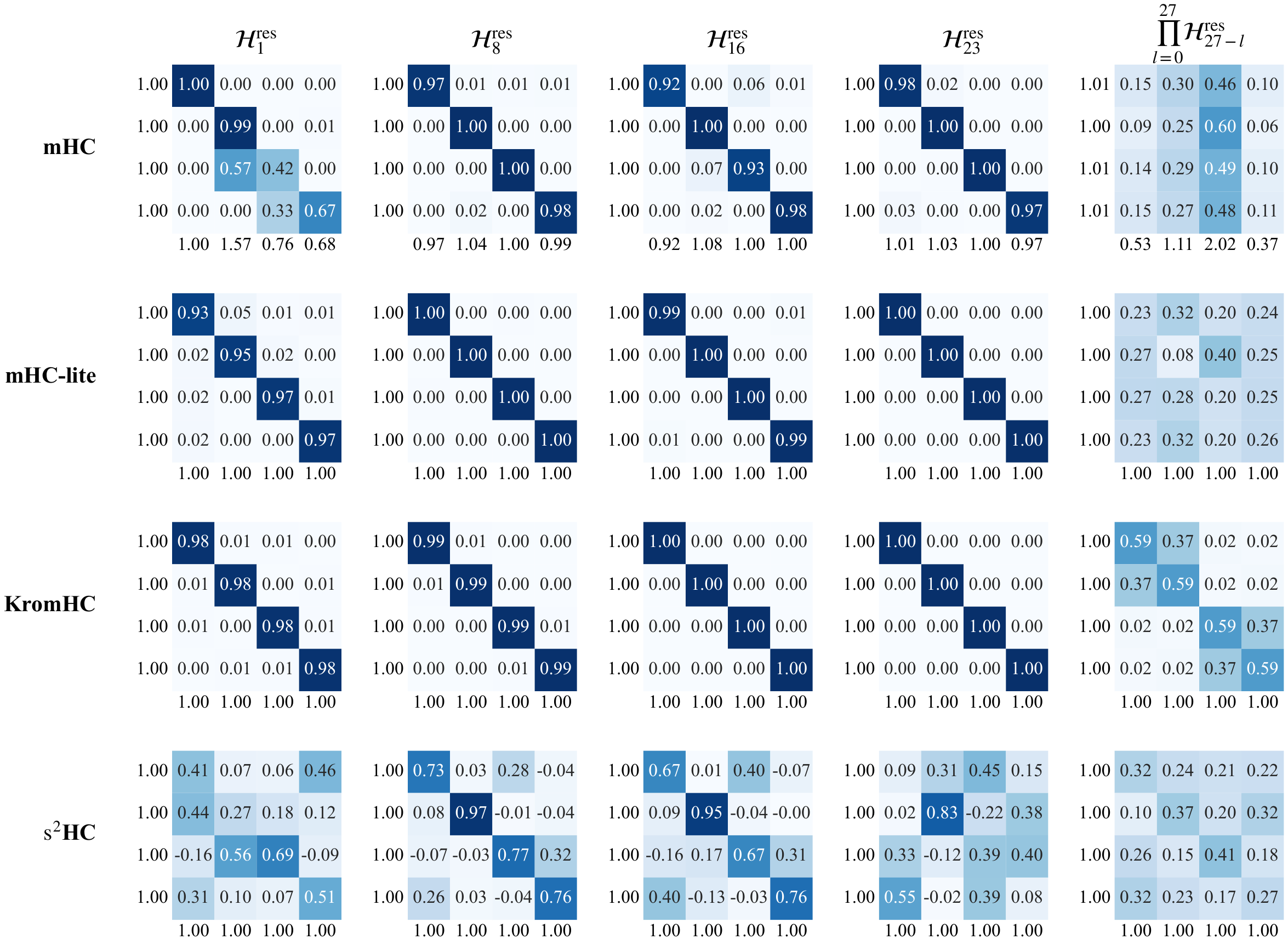}
    \end{minipage}
    \caption{Visualization of layer-wise residual matrices and composite mappings. The figure displays single-layer residual matrices $\mathcal{H}_l^{\mathrm{res}}$ at selected depths ($l \in \{6, 15, 20, 25\}$ for Gemma3 and $l \in \{1, 8, 16, 23\}$ for Qwen2.5), alongside their end-to-end composite mappings ($\prod_{l=0}^{25} \mathcal{H}_{25-l}^{\mathrm{res}}$ for Gemma3 and $\prod_{l=0}^{27} \mathcal{H}_{27-l}^{\mathrm{res}}$ for Qwen2.5) across mHC, mHC-lite, KromHC, and $\mathrm{s}^{2}$HC (Ours). Each matrix is computed by averaging over all tokens within a sequence. The numerical annotations along the vertical and horizontal axes denote the row sums and column sums, respectively.}
    \label{fig:more_heat_map}
\end{figure}

\begin{figure}[t]
    \centering
    \begin{minipage}[c]{0.02\textwidth}
        \centering
        \rotatebox{90}{Gemma3}
    \end{minipage}
    \hspace{1em}
    \begin{minipage}[c]{0.38\textwidth}
        \centering
        \includegraphics[width=\linewidth]{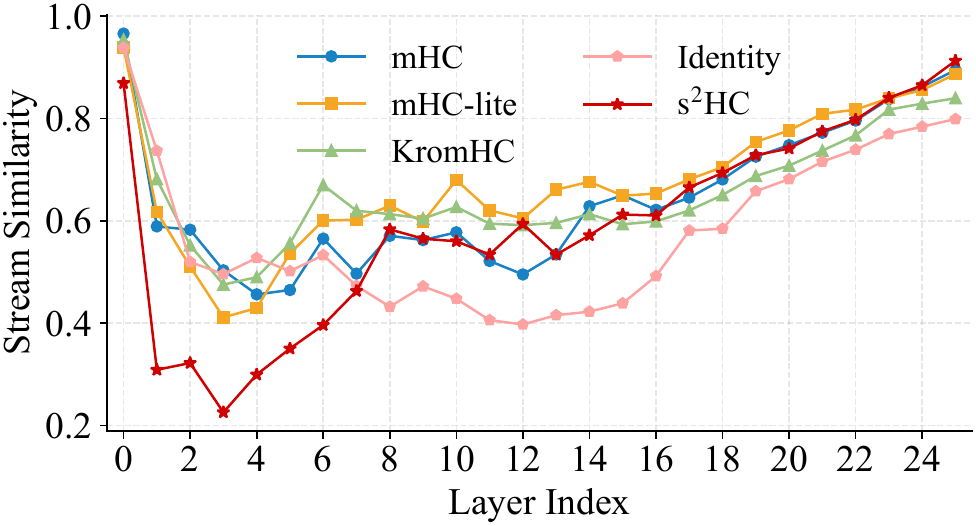}
    \end{minipage}%
    \begin{minipage}[c]{0.37\textwidth}
        \centering
        \includegraphics[width=\linewidth]{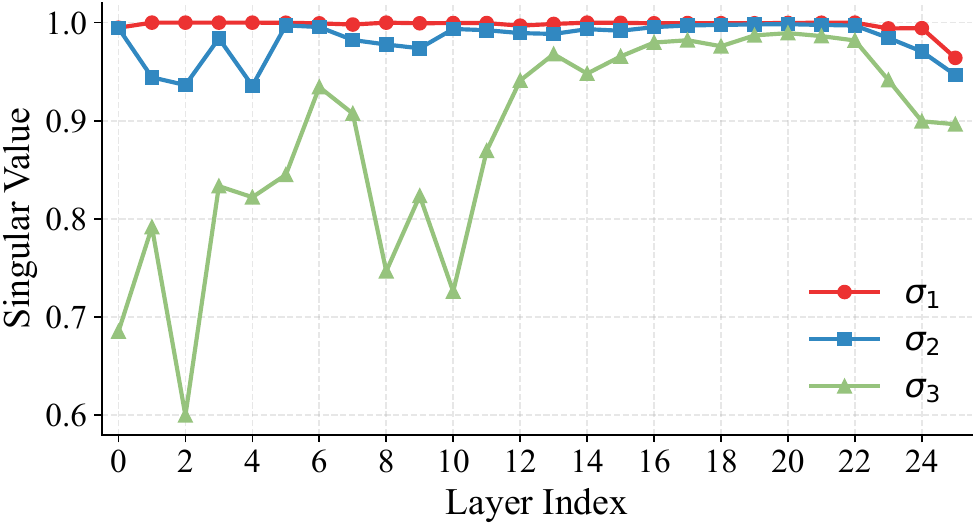}
    \end{minipage}
    
    \begin{minipage}[c]{0.02\textwidth}
        \centering
        \rotatebox{90}{Qwen2.5}
    \end{minipage}
    \hspace{1em}
    \begin{minipage}[c]{0.38\textwidth}
        \centering
        \includegraphics[width=\linewidth]{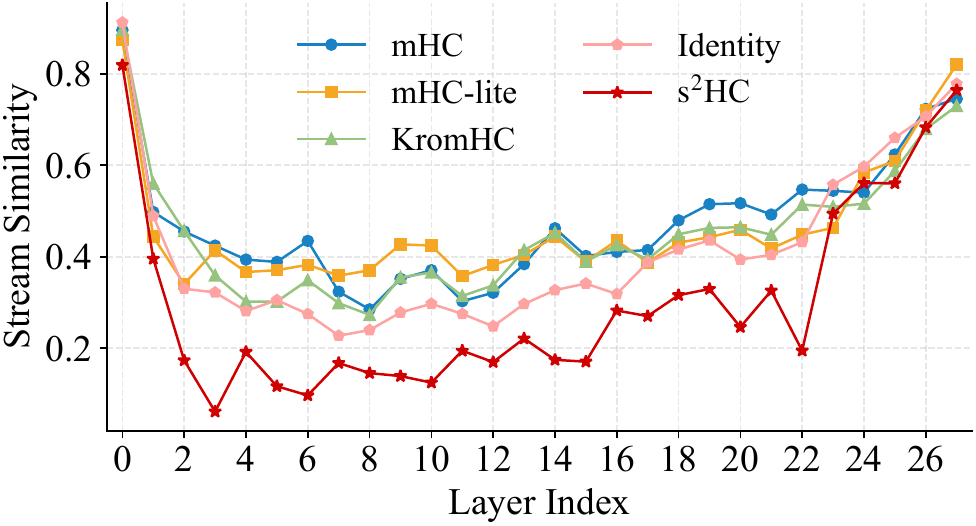}
    \end{minipage}%
    \begin{minipage}[c]{0.37\textwidth}
        \centering
        \includegraphics[width=\linewidth]{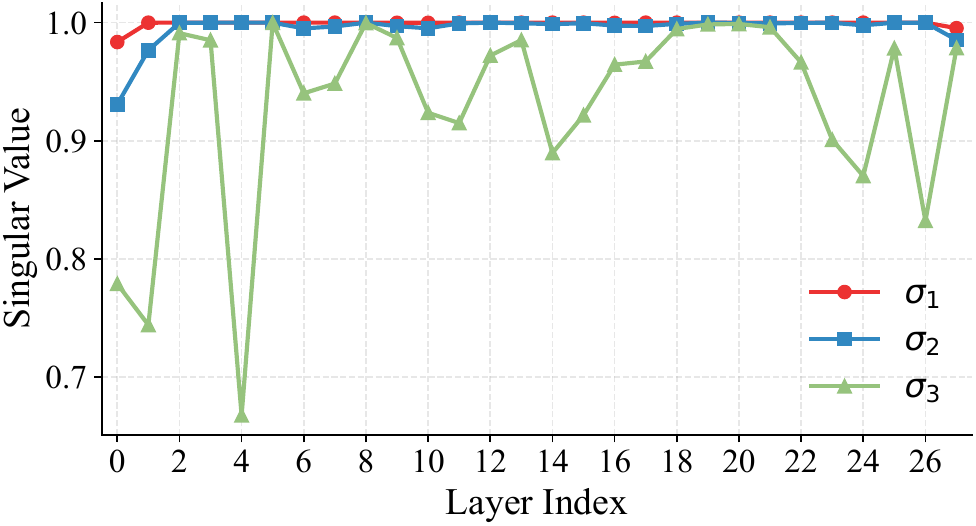}
    \end{minipage}

    \caption{More expressivity analysis of $\mathrm{s}^{2}$HC on Gemma3 and Qwen2.5. \textbf{Left:} Mean pairwise cosine similarity among residual streams after being mixed by $\mathcal{H}_l^{\mathrm{res}}$ across model layers. Identity denotes the baseline where $\mathcal{H}_l^{\mathrm{res}}$ is fixed as an identity matrix. \textbf{Right:} the spectrum of $\mathcal{H}_l^{\mathrm{res}}$ within the subspace orthogonal to the mean component.}
    \label{fig:expressivity_more}
\end{figure}

\begin{figure}[t]
    \centering
    \begin{minipage}[c]{0.02\textwidth}
        \centering
        \rotatebox{90}{Gemma3}
    \end{minipage}
    \begin{minipage}[c]{0.25\textwidth}
        \centering
        \includegraphics[width=\linewidth]{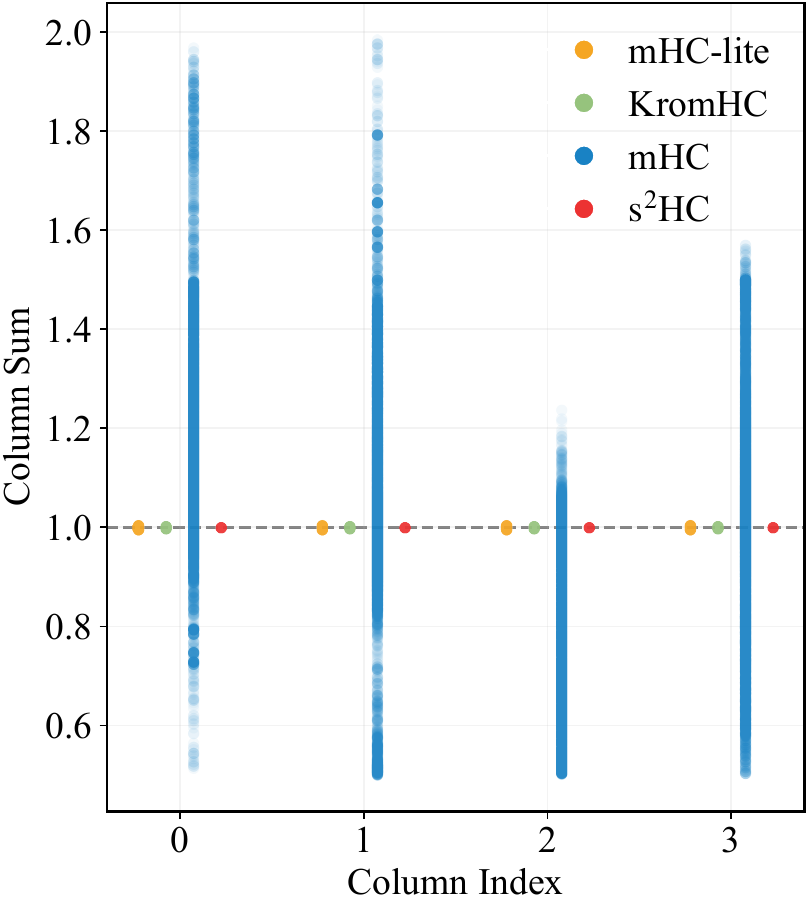}
    \end{minipage}
    \begin{minipage}[c]{0.25\textwidth}
    \centering
        \includegraphics[width=\linewidth]{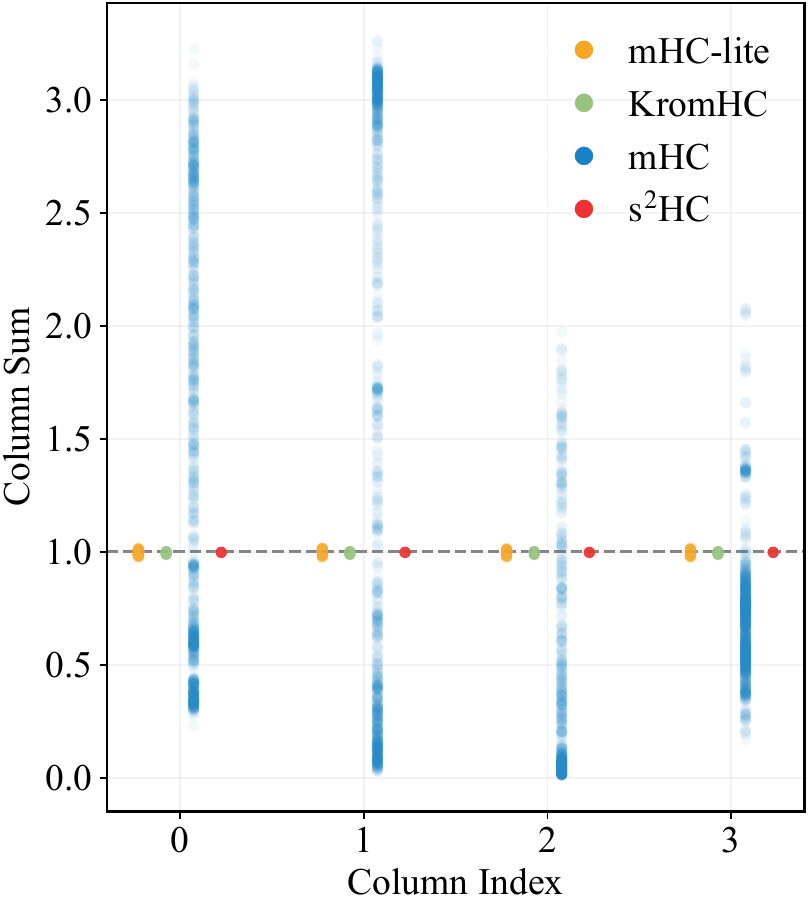}
    \end{minipage}
    \begin{minipage}[c]{0.33\textwidth}
        \centering
        \includegraphics[width=\linewidth]{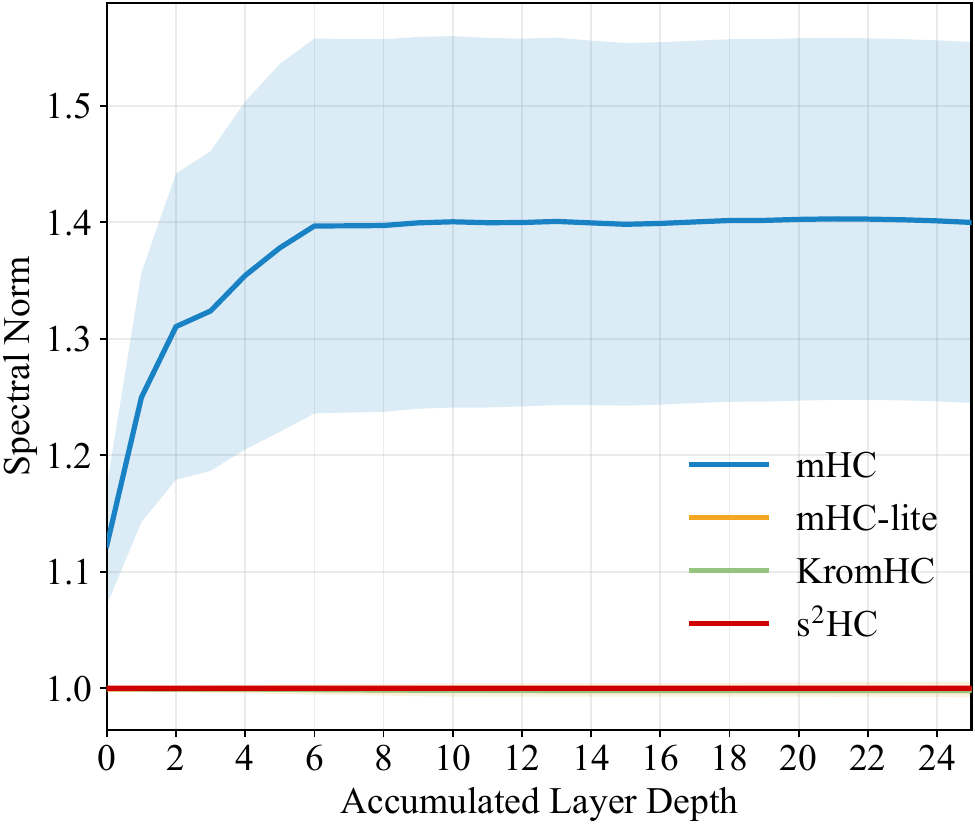}
    \end{minipage}
    
    \begin{minipage}[c]{0.02\textwidth}
        \centering
        \rotatebox{90}{Qwen2.5}
    \end{minipage}
    \begin{minipage}[c]{0.25\textwidth}
        \centering
        \includegraphics[width=\linewidth]{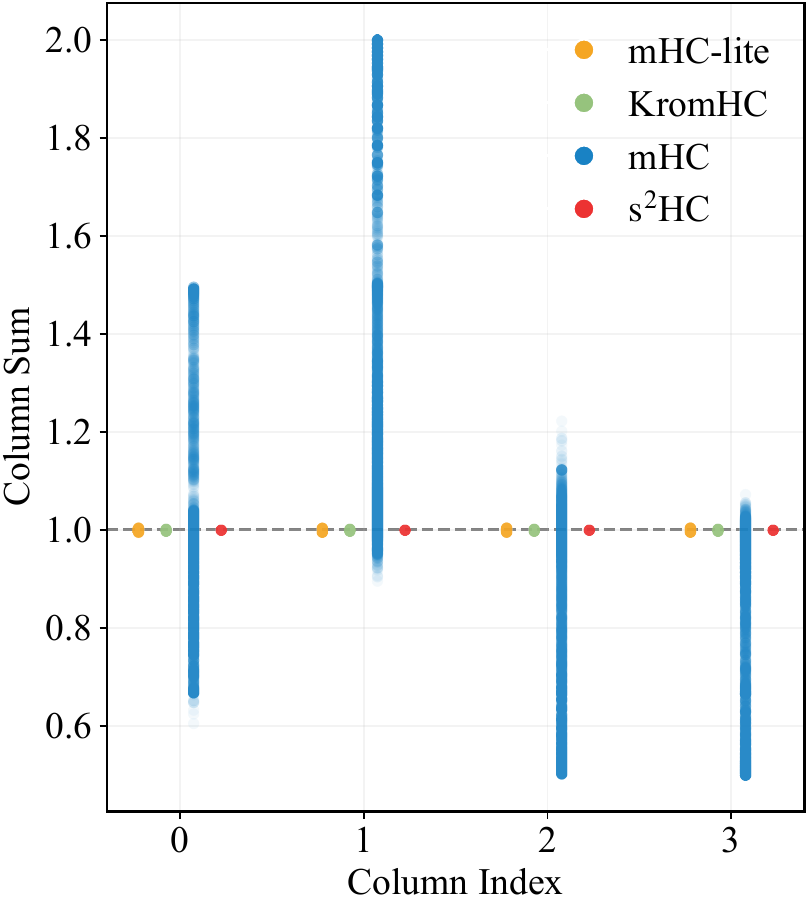}
    \end{minipage}
    \begin{minipage}[c]{0.25\textwidth}
    \centering
        \includegraphics[width=\linewidth]{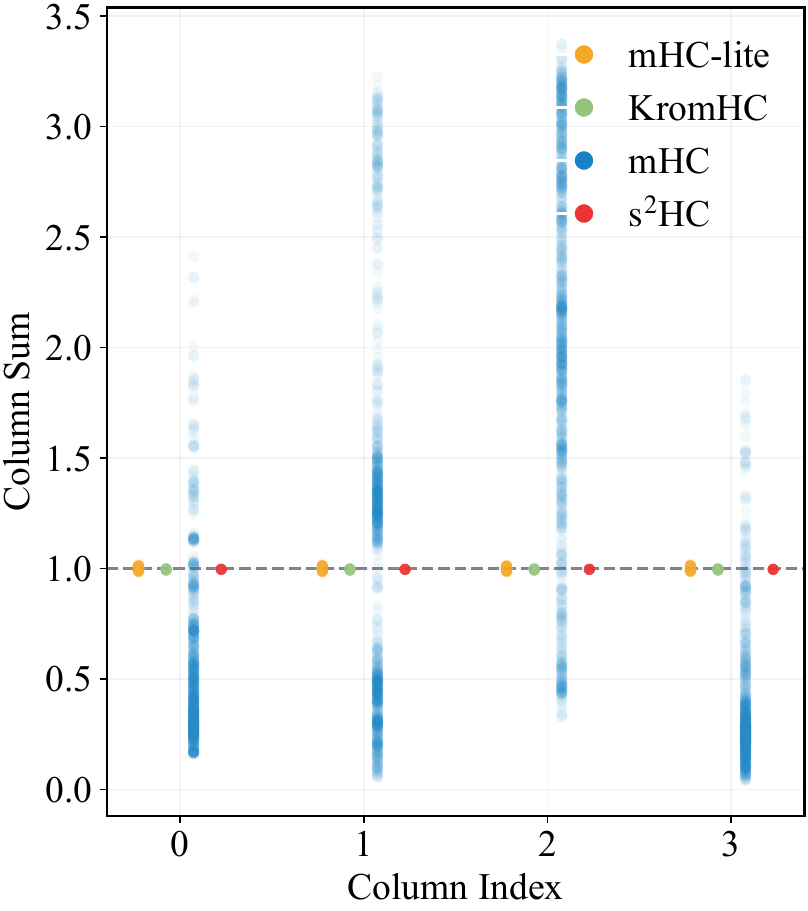}
    \end{minipage}
    \begin{minipage}[c]{0.33\textwidth}
        \centering
        \includegraphics[width=\linewidth]{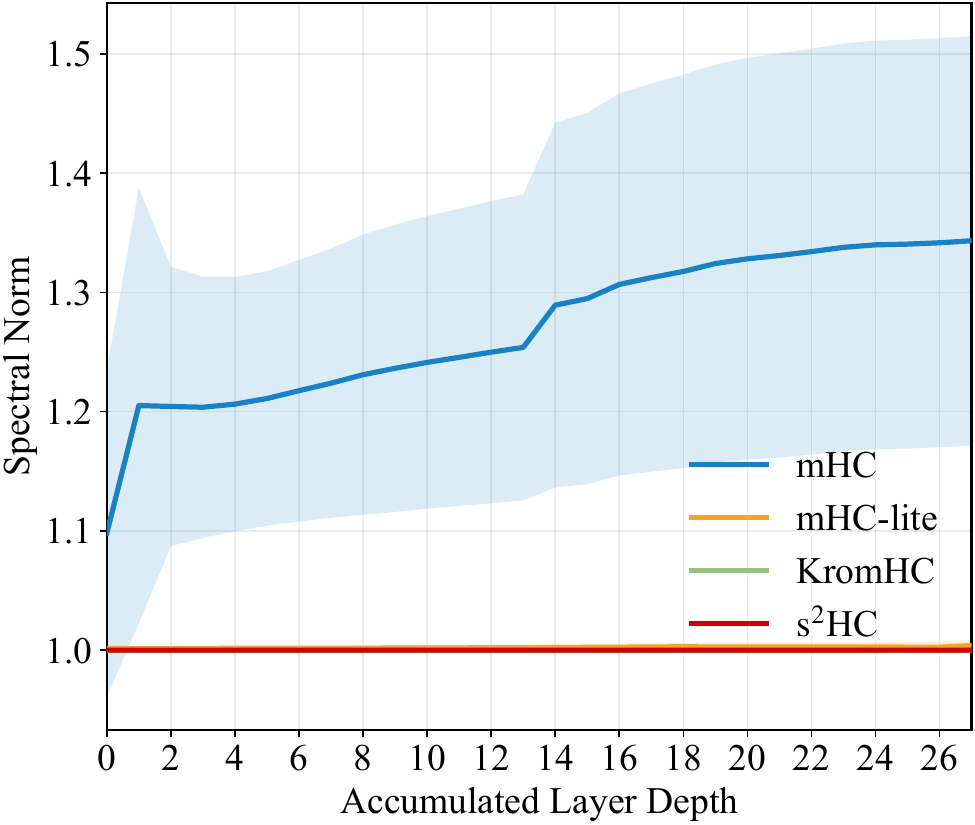}
    \end{minipage}
    \caption{More stability analysis of constrained residual matrices for mHC, mHC-lite, KromHC and $\mathrm{s}^{2}$HC. Statistics are computed from the trained Gemma3 and Qwen2.5 models with 4 residual streams. \textbf{Left:} Column-sum distribution of layer-wise residual matrices $\mathcal{H}_l^\mathrm{res}$ for each layer. \textbf{Middle:} Column-sum distribution of the full-model composite mapping ($\prod_{l=0}^{25}\mathcal{H}_{25-l}^\mathrm{res}$ for Gemma3 and $\prod_{l=0}^{27}\mathcal{H}_{27-l}^\mathrm{res}$ for Qwen2.5). \textbf{Right:} Spectral norm of the cumulative composite mapping ($\|\prod_{l=0}^{L-1}\mathcal{H}_{L-1-l}^\mathrm{res}\|_2$) across the first $L$ layers. Lines denote means across samples while shaded regions indicate standard deviations.}
\end{figure}